\documentclass[twoside,11pt]{article}

\usepackage{blindtext}
\usepackage{amsmath}
\usepackage{tcolorbox}
\usepackage{algorithm}
\usepackage{algorithmic}
\usepackage{amsthm} 
\usepackage{dsfont}
\usepackage{amsfonts} 
\newtheorem{assumption}{Assumption}
\usepackage{natbib}
\usepackage{graphicx}
\usepackage{natbib}
\usepackage{booktabs}
\usepackage{float}
\usepackage{multirow}
\usepackage{subcaption}
\tcbuselibrary{skins}
\tcbset{
    generator_box/.style={
        enhanced,
        colback=yellow!10,
        colframe=yellow!50!black,
        fonttitle=\bfseries,
        title=Generator Prompt,
        attach boxed title to top left={yshift=-\tcboxedtitleheight/2},
        boxed title style={size=small,colback=yellow!50!black,colframe=yellow!50!black},
        coltitle=white,
    },
    discriminator_box/.style={
        enhanced,
        colback=blue!10,
        colframe=blue!50!black,
        fonttitle=\bfseries,
        title=Discriminator Prompt,
        attach boxed title to top left={yshift=-\tcboxedtitleheight/2},
        boxed title style={size=small,colback=blue!50!black,colframe=blue!50!black},
        coltitle=white,
    }
}
\usepackage[preprint]{jmlr2e}

\newcommand{\argmin}{\operatornamewithlimits{arg\,min}}

\usepackage{lastpage}

\begin{document}

\title{Prompt-Dependent Ranking of Large Language Models with Uncertainty Quantification}

\author{\name Angel Rodrigo Avelar Menendez \email aavelarm@g.ucla.edu\\
        \addr Department of Statistics and Data Science, UCLA
        \AND
        \name Yufeng Liu \email yufliu@umich.edu\\
        \addr Department of Statistics, University of Michigan
        \AND
        \name Xiaowu Dai \email daix@ucla.edu\\
        \addr Department of Statistics and Data Science, and of Biostatistics, UCLA}

\maketitle     
\renewcommand{\thefootnote}{\fnsymbol{footnote}}

\begin{abstract}
\noindent
Rankings derived from pairwise comparisons are  central to many economic and computational systems, including model selection, procurement, routing, and reputation mechanisms. In the context of large language models (LLMs), rankings are typically constructed from human preference data and presented as leaderboards that guide deployment decisions. However, existing approaches rely on point estimates, implicitly treating rankings as fixed objects despite substantial estimation noise and context-dependent performance variation. Acting on such rankings can lead to misallocation and welfare loss when apparent differences are not statistically meaningful.
We study prompt-dependent ranking inference under pairwise human preferences and develop a framework for decision-safe rankings with statistically valid uncertainty guarantees. We model preferences using a contextual Bradley–Terry–Luce model in which the latent utility of each model depends on the input prompt. Rather than targeting point estimates of utilities, we directly conduct inference on induced rankings, constructing confidence sets based on simultaneous confidence intervals for pairwise utility differences. This approach yields statistically valid marginal and simultaneous confidence sets for prompt-specific ranks.
Our framework connects recent advances in rank inference to contextual preference learning and provides tools for robust ranking-based decision-making. Empirically, using large-scale human preference data from LLM evaluations, we show that rankings vary substantially across prompt characteristics and that many apparent rank differences are not statistically distinguishable. We further demonstrate how uncertainty-aware rankings identify dominance only when supported by the data and otherwise return partial orders. Overall, our results show that incorporating uncertainty into prompt-dependent ranking for LLMs is essential for reliable economic and computational decisions based on human preference data.
\end{abstract}


\section{Introduction}
\label{sec:intro}
Rankings constructed from pairwise comparisons  play a central role in many  economic and computational systems \citep{ballinger1997decisions, shah2018simple, chau2022spectral}. They are routinely used to guide selection, allocation, routing, and reputation mechanisms in a wide range of applications, including procurement, hiring, admissions, online platforms, and algorithmic benchmarking \citep{goldstein1996league, gupta2002multiple,xie2009confidence, gu2023invidious}. In these settings, rankings are not merely descriptive summaries of observed data but prescriptive objects that map information into concrete actions. As a result, errors in ranking can propagate through downstream decision rules and induce persistent misallocation, distorted incentives, and welfare losses \citep{klein2020joint, andrews2024inference}. In particular, rankings constructed without accounting for statistical uncertainty may appear decisive while being driven by noise, leading decision makers to overcommit to spurious orderings or to discard alternatives that are in fact indistinguishable.

Large language models (LLMs) provide a timely and consequential instance of this phenomenon \citep{chang2024survey,naveed2025comprehensive}. LLMs are commonly evaluated using pairwise human preference data, and the resulting rankings are presented as leaderboards that inform deployment, model routing, and system design \citep{stiennon2020learning, bai2022traininghelpfulharmlessassistant, Xiao02102025, zrnic2025flexible, lu2025}. Despite their widespread influence, existing LLM rankings typically rely on point estimates of latent model quality. Such approaches implicitly treat rankings as fixed and well-identified, even though they are derived from noisy human judgments and finite samples. As a result, many reported rank differences may not be statistically meaningful, yet they are still acted upon in downstream decisions \citep{boubdir2024elo, chiang2024chatbot, chatzi2024prediction}.

A further complication is that LLM performance is inherently context-dependent \citep{dubois2023alpacafarm}. The relative quality of two models can vary substantially across prompts, task types, or query characteristics \citep{zheng2023judging}. Nevertheless, most existing ranking systems assign each model a single global utility, effectively averaging performance across heterogeneous inputs. This practice obscures economically relevant variation and can induce systematic errors when rankings are used to guide prompt-specific decisions, such as routing a query to a particular model or selecting a model for a specialized task  \citep{lu-etal-2024-routing, chen2024routerdc, shnitzer2023largelanguagemodelrouting}. In these applications, acting on point-estimate rankings may lead to unnecessary switching or suboptimal allocation when apparent differences are not statistically supported \citep{chen2024routerdc,ong2024routellm, stripelis2024tensoropera}. In contrast, reliable uncertainty quantification for rankings enables robust selection rules that exploit dominance when it is supported by the data and otherwise avoid overconfident decisions.

This paper studies \emph{prompt-dependent ranking inference} under pairwise human preference feedbacks and develops a framework for constructing decision-safe rankings with valid statistical guarantees. We model preferences using a contextual Bradley–Terry–Luce (BTL) model in which each model’s latent utility depends on the input prompt. Rather than focusing on point estimation of utilities, our inferential target is the ranking induced by these utilities conditional on a given prompt, together with uncertainty measures that are valid for the ranking itself. That is, for a given input prompt, we aim to determine not only which models are preferred, but also how uncertain those rankings are.

Providing statistically valid inference for rankings is fundamentally difficult because ranks are non smooth functionals of latent utilities. Small perturbations in estimated utilities can induce discrete changes in the induced ordering, making ranking inference highly sensitive to estimation errors \citep{bazylik2025finite}. As a consequence, standard approaches that infer ranks indirectly from marginal confidence intervals for individual utilities are generally invalid. A growing literature shows that such procedures can produce confidence sets that are either overly conservative or misleading, with coverage that exceeds or fails to attain the nominal level depending on the underlying utilities \citep{fan2024covariates, fan2023spectral}. The key issue is that rankings are determined by relative comparisons rather than absolute levels. Valid inference therefore requires confidence sets constructed directly for \emph{pairwise utility differences}, which directly determine the ranking \citep{mogstad2024inference}. We adopt this perspective and extend it to a contextual setting, where LLM utilities depend on prompt covariates and the ranking itself varies with covariates.

Our approach constructs simultaneous confidence intervals for prompt-dependent utility differences and uses them to derive marginal and simultaneous confidence sets for LLM ranks. These confidence sets have correct asymptotic coverage and naturally accommodate partial identification: when the data do not support a strict ordering between models, the resulting ranking reflects this uncertainty rather than imposing an arbitrary tie-breaking. From a decision-theoretic perspective, this yields rankings that can be safely used as inputs to downstream mechanisms without overstating evidence. Moreover, our framework allows rankings to vary with the prompt, capturing task- and input-specific performance differences, and directly targets inference on ranks themselves, avoiding the loss of statistical power that arises when rankings are inferred indirectly from utility-wise confidence intervals.

We demonstrate our framework using large-scale human preference data from LLM evaluations. Our empirical results show that rankings vary substantially across prompt characteristics such as length and semantic category, and that many apparent differences between models are not statistically distinguishable once uncertainty is accounted for. At the same time, our method identifies cases of clear dominance, where a model is reliably preferred for a given class of prompts. These findings highlight the limitations of point-estimate leaderboards and demonstrate the value of uncertainty-aware, context-dependent ranking inference for decision-making.

Overall, this paper makes three main contributions.
\begin{itemize}
\item We formalize prompt-dependent ranking of LLMs as a problem of statistical rank inference under a contextual pairwise comparison model, treating rankings as random objects rather than fixed summaries.
\item We develop inference procedures that construct valid marginal and simultaneous confidence sets for prompt-specific ranks based on confidence intervals for utility differences, ensuring correct coverage for the ranking itself.
\item Through empirical analysis of large-scale human preference data, we show how uncertainty-aware rankings alter conclusions drawn from LLM leaderboards and provide a principled foundation for robust ranking-based decisions.
\end{itemize}

The remainder of the paper is organized as follows. Section~\ref{sec:problemformulation} formalizes the prompt-dependent rankings and introduces a contextual BTL model in which rankings are treated as actions taken by a decision maker under uncertainty. Section~\ref{sec:methodology} develops the estimation and inference methodology, including identification, constrained maximum likelihood estimation, and the construction of simultaneous confidence sets for prompt-dependent utility differences and induced rankings. Section~\ref{sec:theory} establishes the theoretical guarantees of the proposed framework, providing asymptotic normality of the estimator, valid coverage for utility and rank confidence sets, and an analysis of ranking behavior under extreme prompt extrapolation. Section~\ref{sec:numericalexp} presents empirical results on large-scale human preference data for LLM evaluation, illustrating how rankings vary with prompt characteristics and how uncertainty-aware inference alters conclusions drawn from point-estimate leaderboards. Section~\ref{sec:relatedwork} describes the related works, and Section~7 concludes the paper with a discussion on future directions.

\section{Problem Formulation}
\label{sec:problemformulation}
We consider a setting in which a decision maker must rank a finite set of LLMs, in order to guide downstream decisions such as selection, routing, or deployment. The decision maker observes noisy and context dependent pairwise human preferences and must construct rankings that are valid for a given prompt while accounting for statistical uncertainty. In this formulation, the prompt represents an observable state of the environment, the ranking is the decision maker’s action, and uncertainty arises from finite and heterogeneous preference data.

\subsection{Contextual Pairwise Preference Model}

There are $M$ LLMs indexed by $m \in [M] := \{1, \ldots, M\}$. Preferences are observed through $L$ pairwise human comparisons. For the $l$th comparison with $l \in [L] = \{1, \dots, L\}$, a pair of models $a_l = (i,j)$ is evaluated under a prompt dependent covariate vector $x_l \in \mathbb{R}^d$, representing observable characteristics of the input prompt, such as length, semantic category, or embedding-based features. The observed outcome
\[
y_{a_l} \in \{0,1\}
\]
indicates whether model $j$ is preferred over model $i$.

Each model $m$ is associated with a prompt dependent latent utility function
\[
\theta_m(x) : \mathbb{R}^d \to \mathbb{R}.
\]
Preferences are modeled using a contextual Bradley--Terry--Luce (BTL) model \citep{bradley1952}. Conditional on the prompt $x$, the probability that model $j$ is preferred to model $i$ is given by
\begin{equation}
\mathbb{P}(y = 1 \mid x, (i,j))
=
\frac{e^{\theta_j(x)}}{e^{\theta_j(x)} + e^{\theta_i(x)}}.
\label{eq:btl}
\end{equation}
Under model \eqref{eq:btl}, rankings are driven by latent utility differences: the probability that model $j$ is preferred to $i$ increases monotonically with $\theta_j(x) - \theta_i(x)$.

\subsection{Prompt Dependent Rankings as Actions}

For a given prompt $x$, the latent rank of model $i$ is defined as
\begin{equation*}
r_i(x)
\;\equiv\;
1 + \sum_{j \in [M]} \mathbf{1}\{\theta_j(x) > \theta_i(x)\},
\label{eq:true_rank}
\end{equation*}
where rank $1$ corresponds to the most preferred model. The vector $r(x) = (r_1(x), \ldots, r_M(x))$ represents the ranking that would be optimal for the decision maker if the latent utilities $\theta_j(x)$'s were known.

In practice, utilities are unobserved and must be inferred from finite data. The decision maker therefore acts on estimated rankings constructed from observed preferences. Rankings are not merely descriptive summaries of the data. They serve as inputs to downstream decision rules, such as selecting a single model, routing prompts to models based on rank, or allocating resources across models \citep{lu-etal-2024-routing, chen2024routerdc, shnitzer2023largelanguagemodelrouting, ong2024routellm}. This makes ranking uncertainty a first order concern.

\subsection{Ranking Inference Targets}

Given a significance level $\alpha \in (0,1)$ and a fixed prompt $x$, the decision maker seeks to construct confidence sets for rankings that explicitly account for estimation uncertainty.
We consider two inferential targets: marginal and simultaneous  rank confidence sets.

\paragraph{Marginal rank confidence sets.}
For a given model $i$, a set $\mathcal{R}_{L,i}(x) \subseteq [M]$ is a marginal confidence set if
\begin{equation}
\lim_{L \to \infty}
\mathbb{P}\bigl(r_i(x) \in \mathcal{R}_{L,i}(x)\bigr)
\;\ge\;
1 - \alpha.
\label{eq:marginal_rank}
\end{equation}
Marginal confidence sets are appropriate when the decision maker is concerned with the position of a particular model, for example whether it belongs to the top $k$ models.

\paragraph{Simultaneous rank confidence sets.}
For any model $i\in[M]$, we have a confidence set $\mathcal{R}^{\mathrm{sim}}_{L,i}(x)\subseteq[M]$. Then a simultaneous confidence set takes the form
\[
\mathcal{R}^{\mathrm{sim}}_L(x)
=
\prod_{i=1}^M \mathcal{R}^{\mathrm{sim}}_{L,i}(x),
\]
and satisfies
\begin{equation}
\lim_{L \to \infty}
\mathbb{P}\bigl(r(x) \in \mathcal{R}^{\mathrm{sim}}_L(x)\bigr)
\;\ge\;
1 - \alpha.
\label{eq:joint_rank}
\end{equation}
where $r(x)=(r_1(x),\ldots,r_M(x))$.
Simultaneous confidence sets are required when decisions depend on the relative ordering of multiple models simultaneously.

These confidence sets naturally induce partial orders when the data do not support a strict ranking. Rather than forcing an arbitrary tie-breaking, they allow the decision maker to exploit statistically supported dominance relationships while avoiding overconfident decisions.

\subsection{Illustrative Example: Ranking LLMs by Prompt Length}
\label{sec:illegpromplength}
\begin{figure}[t]
    \centering
        \begin{subfigure}[t]{0.48\columnwidth}
        \centering
            \includegraphics[width=\linewidth]{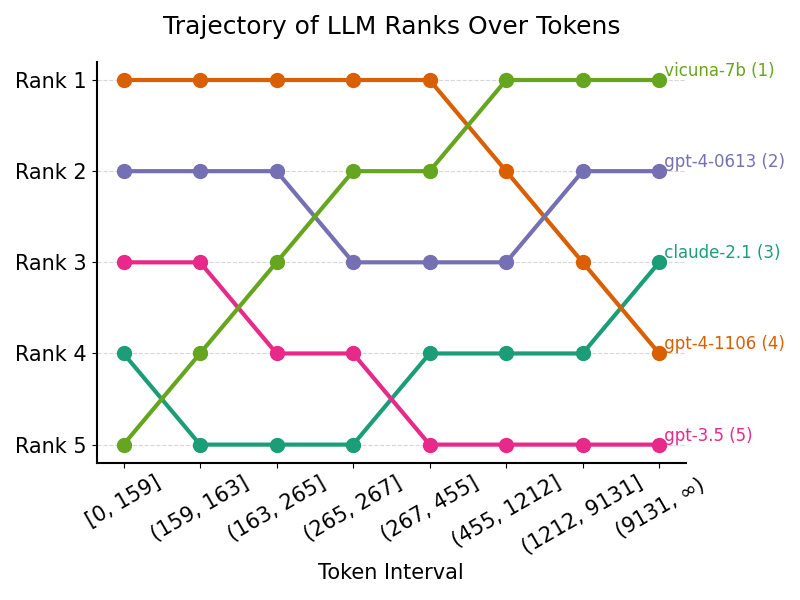}
            \caption{}
            \label{fig:token_length_bump_chart}
        \end{subfigure}
    \hfill
        \begin{subfigure}[t]{0.48\columnwidth}
            \includegraphics[width=\linewidth]{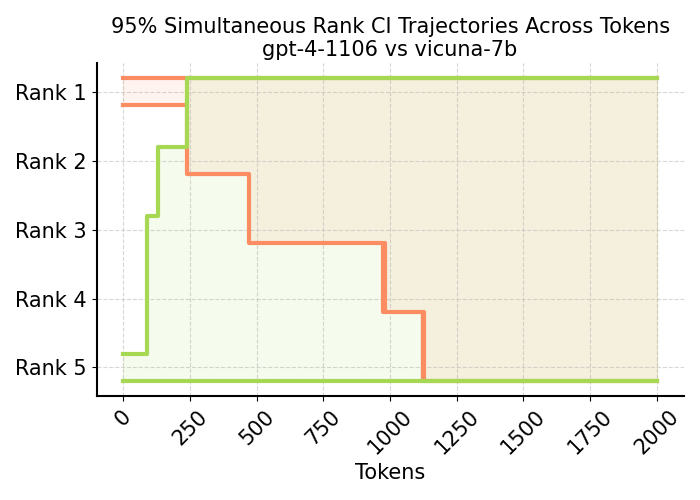}
            \caption{}
            \label{fig:token_length_ranking_ci_vicuna}
        \end{subfigure}
        \caption{
        \textbf{Illustrative example: prompt-dependent rankings and ranking uncertainty.}
        (\subref{fig:token_length_bump_chart})
        Estimated rankings of five LLMs as a function of prompt length, measured by token count.
        Rankings are obtained by ordering the fitted prompt-dependent utilities $\hat{\theta}_i(x)$,
        with rank changes occurring at intersections of the estimated utility functions.
        (\subref{fig:token_length_ranking_ci_vicuna})
        Prompt-dependent 95\% simultaneous confidence set for the rank of Vicuna-7b and GPT-4-1106 as a function of prompt length.
        For short prompts, the rank confidence sets are narrow, indicating statistically supported dominance relationships. As prompt length increases, the confidence sets widen, reflecting growing uncertainty and partial identification of the ranking.
        }
    \label{fig:illustrative_example}
\end{figure}

We illustrate the problem formulation with a concrete example that highlights why prompt dependence and ranking uncertainty are first order concerns for decision making. We consider a dataset of LLM evaluations from the Arena Human Preference dataset \citep{chiang2024chatbot}, which contains approximately 55,000 human annotated pairwise comparisons of LLM responses. Each comparison is associated with an input prompt. We focus on five LLM models: Claude~2.1, GPT-4-1106, GPT-4-0613, GPT-3.5-Turbo-0613, and Vicuna-7b.
Prompt lengths are computed using the BGE-base-en-v1.5 tokenizer \citep{bge_embedding}. 
As a simple and interpretable prompt characteristic, we use the \emph{length} of the prompt measured by the number of tokens. Prompt length is a natural covariate that may affect model performance, as longer prompts can place greater demands on reasoning, context handling, or coherence.

Figure~\ref{fig:illustrative_example} illustrates the estimated rankings as a function of prompt length obtained by ordering the fitted utilities
$\hat{\theta}_i(x)$. The figure shows that the predicted ranking changes as the prompt length increases, with rank reversals occurring at the intersections of the estimated utility functions. From the perspective of the decision maker, this implies that a model that is preferred for short prompts may no longer be preferred for longer prompts. However, point estimate rankings alone are insufficient for guiding decisions. Figure~\ref{fig:illustrative_example} further reports \emph{simultaneous confidence sets} for the ranks of two representative models as a function of prompt length, with the method developed in Section \ref{sec:methodology}. The resulting rank confidence sets widen as prompt length increases, reflecting greater uncertainty
in relative model performance for longer prompts. In particular, beyond a certain range of prompt lengths, the data do not support a statistically meaningful ordering among several models, and the induced ranking becomes only partially identified.

This example illustrates two features of the decision problem studied in this paper. First, rankings are inherently prompt dependent, and ignoring context can obscure economically relevant variation in model performance. Second, ranking uncertainty is unavoidable and must be accounted for explicitly when rankings are used as inputs to downstream decisions. Treating point estimate rankings as fixed objects can lead to overconfident conclusions, whereas uncertainty-aware rankings induce partial orders that more accurately reflect the information contained in the data.

\section{Methodology}
\label{sec:methodology}
This section develops the estimation and inference procedures for the decision
problem in Section~\ref{sec:problemformulation}. The decision maker’s action is a
ranking, and the economically relevant objects are therefore the \emph{induced
ranks} rather than the underlying utility parameters. Our objective is not only
to estimate prompt-dependent utilities, but to provide statistically valid
uncertainty quantification for the rankings implied by those utilities. We first
specify a tractable model for prompt-dependent utilities, then describe estimation from contextual pairwise comparisons, and finally construct
simultaneous confidence sets for utility differences that translate into confidence
sets for ranks.

\subsection{Prompt-Dependent Utility Specification}
We model each model’s latent utility as an explicit function of observed prompt
characteristics. For LLM $i\in[M]$ and prompt covariate $x\in\mathbb{R}^d$, the
utility is
\begin{equation}
\theta_i(x) = \beta_{0i} + x^\top \beta_i,
\label{eq:utility_linear}
\end{equation}
where $\beta_i$ captures how model $i$’s relative performance changes with the prompt $x$ and $\beta_{0i}$ captures intrinsic performance not explained by the prompt.

This specification is motivated by two considerations. First, it is an
interpretable approximation to heterogeneous LLM performance, allowing rankings
to vary systematically with observable context. Second, it yields tractable
estimation and inference in large-scale pairwise comparison settings, which is
essential for constructing confidence sets for rankings. The framework also
admits richer parametric utility functions (e.g., polynomial or spline bases)
under standard regularity conditions \citep{wahba1990spline, dai2023orthogonalized}.

\subsection{Estimation with Pairwise Comparisons}
For each comparison $l\in[L]$, we observe a pair of models $a_l=(i,j)$, a
prompt-dependent covariate vector $x_l$, and a binary preference outcome
$y_{a_l}\in\{0,1\}$ indicating whether model $j$ is preferred to model $i$.
Under the contextual BTL model~\eqref{eq:btl}, preferences depend only on the
utility difference $\theta_j(x_l)-\theta_i(x_l)$. 
To express the likelihood compactly, define the stacked parameter vector
\begin{equation*}
\tilde{\beta}
=
(\beta_{01},\ldots,\beta_{0M}, \beta_1^\top,\ldots,\beta_M^\top)^\top
\in\mathbb{R}^{M+Md},
\label{eq:beta_stacked}
\end{equation*}
and the corresponding design vector
\begin{equation}
\tilde{x}_{a_l}
=
\bigl(e_j,\, e_j\otimes x_l\bigr)^\top
-
\bigl(e_i,\, e_i\otimes x_l\bigr)^\top,
\label{eq:design_vector}
\end{equation}
where $e_i$ denotes the $i$th standard basis vector in $\mathbb{R}^M$.
Then
\begin{equation*}
\theta_j(x_l)-\theta_i(x_l) = \tilde{\beta}^\top \tilde{x}_{a_l}.
\end{equation*}
Let $\sigma(z)=(1+e^{-z})^{-1}$. The negative log-likelihood is
\begin{equation*}
\ell(\tilde{\beta})
=
\sum_{l=1}^L
-\Big[
y_{a_l}\log\sigma(\tilde{\beta}^\top\tilde{x}_{a_l})
+
(1-y_{a_l})\log\bigl(1-\sigma(\tilde{\beta}^\top\tilde{x}_{a_l})\bigr)
\Big].
\label{eq:log_likelihood}
\end{equation*}
This regression problem has the regressors that encode both
the compared models and the prompt.


\subsection{Identification and Ranking}
Only differences in utilities are identified from pairwise choice data. The model
is invariant to additive shifts of the form
$(\beta_{0i},\beta_i)\mapsto(\beta_{0i}+c_0,\beta_i+c_\beta)$ applied uniformly
over $i\in[M]$. To ensure identification, we impose the normalization
\begin{equation}
\sum_{i=1}^M \beta_{0i}=0,
\qquad
\sum_{i=1}^M \beta_i=0.
\label{eq:normalization}
\end{equation}
This fixes a reference utility level without affecting any economically relevant
quantity, since rankings and choice probabilities depend only on utility
differences. 
Let $\Theta$ denote the constrained parameter space induced by
\eqref{eq:normalization}. The constrained MLE is
\begin{equation}
\hat{\beta} = \argmin_{\tilde{\beta}\in\Theta} \ell(\tilde{\beta}).
\label{eq:mle_definition}
\end{equation}
Let $\tilde{\beta}^* \in \Theta$ be the true parameter, defined as 
\begin{equation*}
\tilde{\beta}^* = \argmin_{\tilde{\beta}\in\Theta} \mathbb E[\ell(\tilde{\beta})].
\end{equation*}
\paragraph{\textbf{Constraint representation and projection.}}
Write the constraints as $C\tilde{\beta}=0$. The first constraint is
$\mathbf{1}_M^\top(\beta_{01},\ldots,\beta_{0M})^\top=0$ and the second is the
collection of $d$ constraints
$(I_d\otimes \mathbf{1}_M^\top)(\beta_1^\top,\ldots,\beta_M^\top)^\top=0$.
Equivalently,
\begin{equation*}
C=
\begin{pmatrix}
\mathbf{1}_M^\top & 0\\
0 & I_d\otimes \mathbf{1}_M^\top
\end{pmatrix},
\qquad
C\tilde{\beta}=0.
\label{eq:constraint_matrix}
\end{equation*}
Let $\mathcal{P}$ denote the orthogonal projection onto the feasible subspace
$\{\tilde{\beta}:C\tilde{\beta}=0\}$, given by
\begin{equation}
\mathcal{P} = I - C^\top(CC^\top)^{-1}C.
\label{eq:projection_matrix}
\end{equation}
This projection is used in inference to express asymptotic covariance matrices on the constrained parameter space.

\paragraph{\textbf{Point-estimate ranking as an action.}}
Given covariate $x$, the estimated utility of model $i$ is
\begin{equation*}
\hat{\theta}_i(x)=\hat{\beta}_{0i}+x^\top\hat{\beta}_i.
\end{equation*}
where $\hat{\beta}_{0i}$ and $\hat{\beta}_i$ are the corresponding components of the MLE $\hat{\beta}$ for model $i$ in \eqref{eq:mle_definition}. 
The induced point-estimate rank is
\begin{equation*}
\hat{r}_i(x)
=
1+\sum_{j\in[M]}\mathds{1}\{\hat{\theta}_j(x)>\hat{\theta}_i(x)\}.
\label{eq:rank_def}
\end{equation*}
This is the ranking a decision maker would implement if acting on estimated
utilities alone. Because ranks are non-smooth functionals of the utilities,
small estimation errors in utility differences can change the ordering, motivating inference procedures that directly quantify uncertainty in the induced ranks.


\subsection{Confidence Intervals for Utility Differences}
Rankings are determined by the signs of pairwise utility differences
$\theta_j(x)-\theta_i(x)$. We therefore begin by constructing simultaneous
confidence sets for a collection of utility differences at a fixed covariate
$x$. 

Let $S\subseteq\{(i,j)\in[M]\times[M]:i\neq j\}$ denote the set of ordered
pairs of interest (e.g., all pairs for joint rank inference).
For $(i,j)\in S$, define $\tilde{x}_{ij}$ as in \eqref{eq:design_vector} with
$x_l$ replaced by $x$, so that
$\theta_j(x)-\theta_i(x)=\tilde{\beta}^\top\tilde{x}_{ij}$ and its estimator
is $\hat{\beta}^\top\tilde{x}_{ij}$. Let $\hat{se}_{ij}(x)^2$ denote an estimate of the variance of the predicted
difference, computed as
\begin{equation*}
\hat{se}_{ij}(x)^2
=
\tilde{x}_{ij}^\top \hat{\Sigma} \tilde{x}_{ij},
\end{equation*}
where $\hat{\Sigma}$ is the bootstrap estimate of the covariance matrix of $\hat{\beta}$.
We construct rectangular simultaneous confidence sets using max-type statistics.
Define
\begin{equation}
\label{eq:max_stats}
\begin{aligned}
\mathcal{T}_{\mathrm{lower}}(x)
&\equiv
\max_{(i,j)\in S}
\frac{\tilde{x}_{ij}^\top(\hat{\beta}-\tilde{\beta}^*)}{\hat{se}_{ij}(x)},\\
\mathcal{T}_{\mathrm{upper}}(x)
&\equiv
\max_{(i,j)\in S}
\frac{\tilde{x}_{ij}^\top(\tilde{\beta}^*-\hat{\beta})}{\hat{se}_{ij}(x)},
\\
\mathcal{T}_{\mathrm{symm}}(x)
&\equiv
\max_{(i,j)\in S}
\frac{\bigl|\tilde{x}_{ij}^\top(\hat{\beta}-\tilde{\beta}^*)\bigr|}{\hat{se}_{ij}(x)}.
\end{aligned}
\end{equation}
Let $t_{\boldsymbol{\cdot}}(\alpha,x)$ denote the $(1-\alpha)$ quantile of
$\mathcal{T}_{\boldsymbol{\cdot}}(x)$. These critical values yield one-sided,
two-sided, and equivalence-type rectangular confidence intervals:
\begin{equation}
\label{eq:score_diff_cis}
\begin{aligned}
C_{\mathrm{lower},L}(1-\alpha,S,x)
&=
\prod_{(i,j)\in S}
\Bigl[
\tilde{x}_{ij}(x)^\top\hat{\beta}
-
\hat{t}_{\mathrm{lower}}(\alpha,x)\hat{se}_{ij}(x),
\ \infty
\Bigr), \\
C_{\mathrm{upper},L}(1-\alpha,S,x)
&=
\prod_{(i,j)\in S}
\Bigl(
-\infty,\ 
\tilde{x}_{ij}(x)^\top\hat{\beta}
+
\hat{t}_{\mathrm{upper}}(\alpha,x)\hat{se}_{ij}(x)
\Bigr],\\
C_{\mathrm{symm},L}(1-\alpha,S,x)
&=
\prod_{(i,j)\in S}
\Bigl[
\tilde{x}_{ij}(x)^\top\hat{\beta}
\pm
\hat{t}_{\mathrm{symm}}(\alpha,x)\hat{se}_{ij}(x)
\Bigr],
\\
C_{\mathrm{equiv},L}(1-\alpha,S,x)
&=
C_{\mathrm{lower},L}(1-\alpha/2,S,x)\ \cap\
C_{\mathrm{upper},L}(1-\alpha/2,S,x).
\end{aligned}
\end{equation}
We estimate the critical values $\hat{t}_{\boldsymbol{\cdot}}(\alpha,x)$ using a
parametric bootstrap: we simulate draws from the estimated limiting Gaussian
distribution $\mathcal{N}\!\left(0,\, \tilde{X}_S \hat{\Sigma} \tilde{X}_S^\top \right)$,  and then compute the empirical
$(1-\alpha)$ quantile of the corresponding max statistic in
\eqref{eq:max_stats}. This construction targets simultaneous coverage across all pairs in $S$, which is crucial because rank inference depends on the joint event that many pairwise comparisons are correctly signed.

\begin{definition}
The ranking between models $j$ and $k$ is \emph{statistically resolved} at level $\alpha$ if the simultaneous confidence interval in \eqref{eq:score_diff_cis} for the utility difference $\theta_j(x)-\theta_k(x)$ excludes zero. If the confidence interval contains zero, the relative ranking is said to be \emph{statistically unresolved}.
\end{definition}

\subsection{Confidence Sets for Ranks}
\label{subsec:rank_uq}
We now translate confidence intervals for utility differences into confidence sets for ranks. Following \citet{mogstad2024inference}, we use rectangular confidence intervals for utility differences to infer which pairwise orderings are \emph{statistically resolved}. 
Assume we have a rectangular confidence interval $C_L(1-\alpha, S, x)$ such that 
\begin{equation}
    C_L(1-\alpha, S, x) = \prod_{(j,k)\in S} C_L(1-\alpha, (j,k), x)
    \label{eq:score_difference_rectangular_assumption}
\end{equation}
and have asymptotic coverage  for the utility differences
$\Delta_S(x)=\{\theta_j(x)-\theta_k(x):(j,k)\in S\}$,
\begin{equation}
    \lim_{L\rightarrow\infty} P\{\Delta_S(x) \in C_L(1-\alpha, S, x)\} \geq 1-\alpha.
    \label{eq:score_difference_coverage_assumption}
\end{equation}
For any ordered pair $(j,k)$, if $C_L(1-\alpha,(j,k),x)\subset(0,\infty)$, then $j$ ranked above $k$ is statistically resolved; if $C_L(1-\alpha,(j,k),x)\subset(-\infty,0)$, then $j$  ranked below $k$ statistically resolved; and if the interval contains zero, the relative ranks of $(j,k)$ is statistically unresolved. 

Note that the confidence intervals defined in \eqref{eq:score_diff_cis} are rectangular by construction and therefore satisfy \eqref{eq:score_difference_rectangular_assumption}. Moreover, in
Section~\ref{sec:asymputility} we show that, under mild regularity
conditions, these intervals also satisfy the asymptotic coverage
requirement in \eqref{eq:score_difference_coverage_assumption}.

\paragraph{\textbf{Marginal rank confidence set.}}
To infer the rank of a single LLM $j$, set
$S_j=\{(j,k):k\in[M]\setminus\{j\}\}$ and use $C_L(1-\alpha,S_j,x)$ that satisfies \eqref{eq:score_difference_rectangular_assumption} and \eqref{eq:score_difference_coverage_assumption}. Define
\begin{equation*}
\begin{aligned}
N_j^+(x)& =\{k\neq j:\ C_L(1-\alpha,(j,k),x)\subset(0,\infty)\},\\
N_j^-(x)& =\{k\neq j:\ C_L(1-\alpha,(j,k),x)\subset(-\infty,0)\}.
\end{aligned}
\end{equation*}
The set $N_j^+(x)$ contains items that are statistically dominated by $j$ at
prompt $x$, while $N_j^-(x)$ contains items that statistically dominate $j$.
All remaining comparisons are statistically unresolved. This yields the marginal rank set
\[
\mathcal R_{L,j}(x)=\{|N_j^-(x)|+1,\ldots,M-|N_j^+(x)|\}.
\]
This set is informative when many pairwise orderings are statistically resolved, and naturally widens when the data do not support a strict ordering.

\paragraph{\textbf{Simultaneous rank confidence set.}}
For simultaneous inference on all ranks, set
$S_{\mathrm{sim}}=\{(j,k)\in[M]\times[M]:j\neq k\}$ and use $C_L(1-\alpha,S_{\mathrm{sim}},x)$ that satisfies \eqref{eq:score_difference_rectangular_assumption} and \eqref{eq:score_difference_coverage_assumption}.
Define, for each $j$,
\begin{equation*}
\begin{aligned}
N_{j,\mathrm{sim}}^+(x)& =\{k\neq j:\ C_L(1-\alpha,(j,k),x)\subset(0,\infty)\},\\
N_{j,\mathrm{sim}}^-(x) & =\{k\neq j:\ C_L(1-\alpha,(j,k),x)\subset(-\infty,0)\}.
\end{aligned}
\end{equation*}
The corresponding simultaneous rank set for item $j$ is
\[
\mathcal R_{L,j}^{\mathrm{sim}}(x)
=
\{|N_{j,\mathrm{sim}}^-(x)|+1,\ldots,M-|N_{j,\mathrm{sim}}^+(x)|\}.
\]

\section{Theoretical Guarantees}
\label{sec:theory}
This section develops the theoretical guarantees for contextual ranking
inference. We first establish asymptotic normality of the constrained maximum
likelihood estimator in the contextual BTL model under fixed
numbers of models and covariates. We then derive valid rectangular confidence
sets for prompt-dependent utility differences and show how these sets induce
marginal and simultaneous confidence intervals for ranks. Finally, we analyze the asymptotic behavior of rankings and their confidence sets under extreme prompt extrapolation, characterizing regimes in which contextual effects dominate intrinsic utilities.

\subsection{Asymptotic Normality of the Utility Estimation}
\label{sec:asymputility}
We begin by analyzing the asymptotic distribution of the constrained maximum
likelihood estimator $\hat{\beta}$ defined in \eqref{eq:mle_definition}.
All limits are taken as the number of pairwise comparisons $L \to \infty$,
with the number of models $M$ and the covariate dimension $d$ held fixed.

The following assumptions ensure identifiability and regularity of the
estimation problem.

\begin{assumption}[Comparison Graph Connectivity]
\label{as:comparison_graph_connectivity}
Let $\mathcal{G}$ be the comparison graph with vertex set $\mathcal{V}=[M]$ and edge set $\mathcal{E}$, where $(i,j)\in\mathcal{E}$ if the pair $(i,j)$ is observed with positive probability. We assume $\mathcal{G}$ is connected.
\end{assumption}

\begin{assumption}[Covariate Richness]
\label{as:design_matrix_rank}
Let $\bar{x}_l = (e_j \otimes x_l) - (e_i \otimes x_l)$ for comparison
$a_l=(i,j)$, and let $\bar{X}$ denote the matrix with rows $\bar{x}_l^\top$.
We assume $\bar{X}$ has full column rank.
\end{assumption}

\begin{assumption}[Boundedness]
\label{as:boundedness}
The design vectors satisfy
$\sup_{l\in[L]} \|\tilde{x}_{a_l}\|_2 < B_2 < \infty$.
Moreover, utilities are uniformly bounded:
$\sup_{l\in[L]} |\tilde{\beta}^{*\top}\tilde{x}_{a_l}| < \infty$.
\end{assumption}

\begin{assumption}[Asumptotic Fisher Information]
\label{as:limit_fisher}
Let
\begin{equation*}
        \mathcal{I}_L(\tilde{\beta}) = \sum_{l=1}^{L}\sum_{a\in\mathcal{E}}p_{a} \sigma(\tilde{\beta}^\top \tilde{x}_{a_l})(1-\sigma(\tilde{\beta}^\top \tilde{x}_{a_l}))\tilde{x}_{a_l}\tilde{x}_{a_l}^\top.
    \end{equation*}
We assume the normalized Fisher information converges as $L\to\infty$:
\[
\frac{1}{L}\mathcal{I}_L(\tilde{\beta}^*)
\;\longrightarrow\;
\bar{\mathcal{I}},
\]
where $\bar{\mathcal{I}}$ is positive definite on the constrained parameter space $\Theta$ induced by \eqref{eq:normalization}.
\end{assumption}

We remark that Assumptions \ref{as:comparison_graph_connectivity}-\ref{as:limit_fisher} are mild and  standard in the literature. Specifically, Assumption \ref{as:comparison_graph_connectivity} ensures global identifiability of relative utilities across all models. If the comparison graph were disconnected, utilities would only be identified up to independent shifts within each connected component, rendering cross component comparisons impossible \citep{chen2022partial, gao2023uncertainty, fan2024covariates, fan2025ranking}. Assumption \ref{as:design_matrix_rank} requires sufficient variation in prompt characteristics across comparisons to identify heterogeneous, prompt dependent performance effects across LLMs. This condition parallels standard rank and full support assumptions in contextual discrete choice and ranking models \citep{fan2024covariates}. Assumption \ref{as:boundedness} guarantees finite moments of the utility and Hessian of the log likelihood, which is a routine regularity condition in fixed design logistic regression and contextual BTL models \citep{fan2024covariates, Han08012026, dong2025statistical}. Finally, Assumption \ref{as:limit_fisher} allows for fixed and potentially non random designs, a setting that is well suited to human evaluated LLM benchmarks where prompts and comparison frequencies are chosen by the experimenter rather than drawn randomly. Under these conditions, the constrained maximum likelihood estimator defined in \eqref{eq:mle_definition} admits an asymptotic normal representation.

\begin{theorem}
\label{thm:mle_asymptotic_normality}
Suppose Assumptions
\ref{as:comparison_graph_connectivity}--\ref{as:limit_fisher} hold.
Then as $L\to\infty$,
\[
\sqrt{L}\bigl(\hat{\beta}-\beta^*\bigr)
\;\xrightarrow{d}\;
\mathcal{N}\!\left(0,\,
(\mathcal{P}\bar{\mathcal{I}}\mathcal{P})^\dagger\right),
\]
where $\mathcal{P}$ is defined in \eqref{eq:projection_matrix} as the projection onto the constrained parameter space, and  $\bar{\mathcal{I}}$ is the limiting Fisher information defined in Assumption \ref{as:limit_fisher}.
\end{theorem}
\noindent
We make two remarks on Theorem \ref{thm:mle_asymptotic_normality}. First, the appearance of the projection operator $\mathcal{P}$ reflects the fundamental non-identifiability of utilities up to additive shifts across LLM models and prompt covariates. 
Second, our analysis accommodates prompt-dependent covariates that vary across comparisons and directly affect relative utilities in \eqref{eq:utility_linear}. In contrast, \citet{fan2024covariates} study a BTL
model that is applicable to LLM evaluations with LLM-specific covariates and constant prompts, and focusing on asymptotic regimes in which the number of LLMs diverges. In contrast, our framework is more realistic and allows for a fixed set of LLMs and prompt-specific covariates, aligning with empirical evaluation settings where both prompts and models vary across comparisons.

Rankings depend only on pairwise utility differences.
Accordingly, our inferential target is the vector of utility differences
for a given prompt $x$ and comparison set $S$.
Let $\tilde{X}_S$ collect the design vectors $\tilde{x}_{ij}^\top$ for
$(i,j)\in S$ as in \eqref{eq:design_vector}. Let
$\Delta_S(x) = \tilde{X}_S\beta^*$ and 
$\hat{\Delta}_S(x) = \tilde{X}_S\hat{\beta}$.
\begin{corollary}
\label{cor:score_difference_asymptotic_normality}
Under the conditions of Theorem~\ref{thm:mle_asymptotic_normality},
\[
\sqrt{L}\bigl(\hat{\Delta}_S(x)-\Delta_S(x)\bigr)
\;\xrightarrow{d}\;
\mathcal{N}\!\left(0,\,
\tilde{X}_S(\mathcal{P}\bar{\mathcal{I}}\mathcal{P})^\dagger\tilde{X}_S^\top
\right).
\]
\end{corollary}
\noindent
Corollary~\ref{cor:score_difference_asymptotic_normality} provides the asymptotic
distribution required to construct valid confidence sets for collections of
utility differences.
We now show that the rectangular confidence intervals defined in
\eqref{eq:score_diff_cis} satisfy the desired coverage
property \eqref{eq:score_difference_coverage_assumption}.

\begin{theorem}
\label{thm:score_differences_ci_coverage}
Under the conditions of Theorem~\ref{thm:mle_asymptotic_normality}, let
$C_{\boldsymbol{\cdot},L}(1-\alpha,S,x)$ be defined as in
\eqref{eq:score_diff_cis}.
Then,
\begin{equation*}
    \lim_{L\rightarrow\infty} P\{\Delta_S(x) \in C_{\boldsymbol{\cdot},L}(1-\alpha, S, x)\} \geq 1-\alpha.
\end{equation*}
That is, the rectangular confidence intervals $C_{\boldsymbol{\cdot},L}(1-\alpha,S,x)$ satisfy the coverage
property \eqref{eq:score_difference_coverage_assumption}
\end{theorem}

\subsection{Statistical Inference for Ranks}

We now translate inference on utility differences into inference on rankings. This step is nontrivial because ranks are non-smooth, discontinuous
functionals of latent utilities: arbitrarily small perturbations of utilities can change the induced ordering.   Given a prompt-dependent covariate $x$, we use rectangular confidence sets for utility differences that satisfy
\eqref{eq:score_difference_rectangular_assumption} and
\eqref{eq:score_difference_coverage_assumption}. These sets allow us to identify which pairwise orderings are statistically resolved and which remain
indeterminate, and to aggregate this information into confidence sets for
ranks.

\begin{theorem}
\label{thm:rank_marginal_ci_coverage}
Let $S_j=\{(j,k):k\neq j\}$. If $C_L(1-\alpha,S_j,x)$ satisfies
\eqref{eq:score_difference_rectangular_assumption} and
\eqref{eq:score_difference_coverage_assumption}, then
\[
\lim_{L\rightarrow\infty}
P\bigl(r_j(x) \in \mathcal{R}_{L,j}(x)\bigr) \geq 1-\alpha,
\]
where
\[
\mathcal{R}_{L,j}(x)
=
\{|N_j^-(x)|+1,\dots,M-|N_j^+(x)|\}
\subseteq \{1,\ldots,M\}.
\]
Hence, $\mathcal{R}_{L,j}(x)$ is a valid marginal confidence set for the rank of
item $j$ at covariate $x$.
\end{theorem}
\noindent
Theorem \ref{thm:rank_marginal_ci_coverage} shows that the rank of item $j$ is partially identified whenever
some pairwise orderings involving $j$ are statistically unresolved. The width
of the confidence set directly reflects the amount of information available in
the data: it shrinks when many pairwise comparisons are resolved and expands
when evidence is weak.

\begin{theorem}
\label{thm:rank_simultaneous_ci_coverage}
Let $S_{\mathrm{sim}}=\{(j,k)\in[M]\times[M]:j\neq k\}$. If
$C_L(1-\alpha,S_{\mathrm{sim}},x)$ satisfies
\eqref{eq:score_difference_rectangular_assumption} and
\eqref{eq:score_difference_coverage_assumption}, then
\[
\lim_{L\rightarrow\infty}
P\left\{\bigcap_{j\in[M]}
\bigl\{r_j(x) \in \mathcal R_{L,j}^{\emph{joint}}(x)\bigr\}\right\}
\geq 1-\alpha,
\]
where
\[
\mathcal R_{L,j}^{\emph{joint}}(x)
=
\{|N_{j,\mathrm{sim}}^-(x)|+1,\dots,M-|N_{j,\mathrm{sim}}^+(x)|\}.
\]
Consequently, the simultaneous confidence set
\[
\mathcal{R}^{\emph{joint}}_L(x)
=
\prod_{j=1}^M \mathcal{R}^{\emph{joint}}_{L,j}(x)
\]
satisfies the simultaneous coverage guarantee in
\eqref{eq:joint_rank}.
\end{theorem}
\noindent
Both the marginal and simultaneous coverage gaurantees in Theorem \ref{thm:rank_marginal_ci_coverage} and \ref{thm:rank_simultaneous_ci_coverage} extend the rank inference
framework of \citet{mogstad2024inference} to a contextual setting, in which
utilities and rankings depend on input prompt characteristics. This extension is essential for applications such as LLM evaluation, where rankings vary systematically across prompts and must be interpreted conditionally on the input.

\subsection{Asymptotic Behavior Under Extreme Prompts}
We analyze how rankings and their associated confidence sets behave as prompt covariates grow large along a fixed direction. This asymptotic regime provides insight into how contextual effects dominate intrinsic performance and clarifies the limits of extrapolating rankings beyond the support of observed prompts.

Formally, fix a direction $v \in \mathbb{R}^d$ and consider a sequence of covariates of the form $x = \lambda v$ with $\lambda \to \infty$. This setting captures extreme prompts that increasingly emphasize a particular feature dimension, such as very long prompts or highly specialized task attributes.

\begin{proposition}
\label{prop:asymptotic_extrapolation}
Let the estimated utilities be
$\hat{\theta}_i(x) = \hat{\beta}_{0i} + x^\top \hat{\beta}_i$ for $i \in [M]$.
Fix a direction $v \in \mathbb{R}^d$ and set $x = \lambda v$. Assume that
$v^\top \hat{\beta}_i \neq v^\top \hat{\beta}_j$ for all $i \neq j$.
Then, as $\lambda \to \infty$:
\begin{itemize}
    \item The predicted rank converges to the ordering induced solely by the
    projected covariate effects:
    \[
        \hat{r}_j(\lambda v)
        \;\longrightarrow\;
        1 + \sum_{i \neq j}
        \mathds{1}\!\left\{v^\top \hat{\beta}_i > v^\top \hat{\beta}_j\right\}.
    \]
    \item The normalized utility-difference confidence sets
    $C_\lambda(1-\alpha,\lambda v,(i,j)) = \lambda^{-1} C(1-\alpha,\lambda v,(i,j))$
    converge almost surely to
    \[
        C_\infty(1-\alpha,S,v)
        =
        \prod_{(i,j)\in S}
        \left[
        v^\top(\hat{\beta}_i - \hat{\beta}_j)
        \pm
        t^{\infty}_{\alpha/2}
        \sqrt{
        v^\top
        \widehat{\mathrm{Var}}(\hat{\beta}_i - \hat{\beta}_j)
        v
        }
        \right],
    \]
    where $t^{\infty}_{\alpha/2}$ denotes the $(1-\alpha/2)$ quantile of the
    limiting max statistic $\mathcal{T}_\infty$.
    \item The marginal and joint rank confidence sets $R_j(\lambda v)$ converge
    to limiting sets $R_j^\infty$ constructed from $C_\infty(1-\alpha,S,v)$ via
    the same rules as in
    Theorems~\ref{thm:rank_marginal_ci_coverage} and
    \ref{thm:rank_simultaneous_ci_coverage}.
\end{itemize}
\end{proposition}
\noindent
Proposition~\ref{prop:asymptotic_extrapolation} shows that, under extreme prompts, intrinsic utility components $\hat{\beta}_{0i}$ become asymptotically negligible, and rankings are governed entirely by the covariate-specific effects $\hat{\beta}_i$. Importantly, the limiting behavior of the confidence sets depends on whether projected differences $v^\top(\hat{\beta}_i-\hat{\beta}_j)$ are statistically distinguishable. When these differences are small relative to their estimation uncertainty, the limiting confidence sets remain wide, reflecting fundamental uncertainty that cannot be resolved by extrapolation.

This phenomenon has direct implications for LLM evaluation. In our illustrative example in Section \ref{sec:illegpromplength}, taking $v = \mathbf{1}$ and interpreting $\lambda$ as prompt length, the estimated covariate differences $\hat{\beta}_i - \hat{\beta}_j$ are not statistically different for any pair of models. Consequently, $0 \in C_\infty(1-\alpha,(i,j),v)$ for all $(i,j)$, and the asymptotic rank confidence sets degenerate to the uninformative range $[1,M]$, where $M=5$. 

\section{Numerical Experiments}
\label{sec:numericalexp}
This section evaluates the empirical implications of prompt-dependent ranking inference using large-scale human preference data. Our objective is not to optimize predictive accuracy, but to examine how accounting for prompt characteristics and ranking uncertainty alters the conclusions a decision maker would draw from LLM leaderboards.  Code for reproducing all empirical results is available at the anonymous repository at \url{https://anonymous.4open.science/r/EC-2026-575/}.

\subsection{Prompt Categories}
\begin{figure}[t]
    \vskip 0.2in
  \begin{center}
  \centerline{\includegraphics[width=0.7\columnwidth]{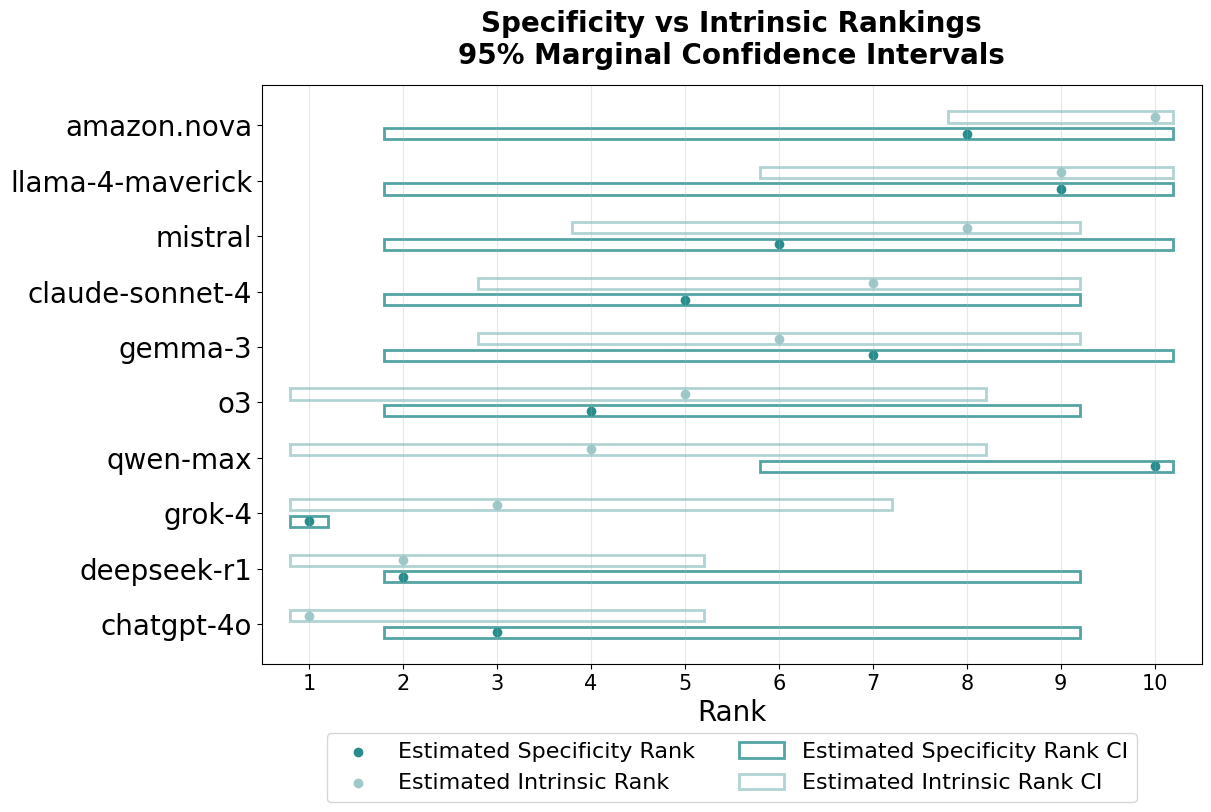}}
    \caption{\textbf{Prompt-dependent rankings and ranking uncertainty for the Specificity category.}
    Predicted rankings and 95\% marginal confidence intervals for ten LLMs under
    intrinsic preferences and under prompts exclusively labeled with the
    Specificity category.  Intrinsic estimates, corresponding to a zero covariate vector, exhibit wide confidence intervals, indicating substantial uncertainty in aggregate
    preferences.  Introducing the Specificity category alters both predicted ranks and their
    uncertainty.  While many apparent rank differences remain statistically insignificant,
    Grok-4 exhibits statistically supported dominance with a singleton confidence
    interval.  The figure illustrates how prompt characteristics affect both rankings and
    their reliability, and why uncertainty-aware rankings are essential for decision making.}
    \label{fig:specificity_intrinsic_results}
    \end{center}
\end{figure}

\paragraph{\textbf{Dataset and setup.}}
We apply our framework to the Arena Human Preference 140k dataset
\citep{chiang2024chatbot}, which contains approximately 140,000 human-annotated
pairwise comparisons of LLM responses. Each comparison is associated with a
user-provided prompt and a binary preference outcome. We focus on a subset of ten
widely used models: Amazon Nova-pro, ChatGPT-4o, Claude Sonnet-4, DeepSeek-R1,
Grok-4, Llama-4 Maverick, Mistral Small, O3, Qwen-Max, and Gemma-3.

The dataset provides ten prompt category tags: Code, Creative Writing,
Complexity, Creativity, Domain Knowledge, Problem Solving, Real World,
Specificity, Technical Accuracy, and Math. We encode these categories as a binary
covariate vector $x \in \{0,1\}^{10}$, where each entry indicates whether the
prompt reflects a given category. For example, the prompt \textit{''List 5 catchy and well-thought brand names that "combine" Lego and resume building''} is associated with the tags \emph{Domain Knowledge}, \emph{Specificity}, and \emph{Technical Accuracy}, corresponding to the covariate vector $(0, 0, 0, 0, 1, 0, 0, 1, 1, 0)^\top$.  Multi-turn conversations are decomposed into independent pairwise comparisons, and ties are excluded to ensure compatibility with the BTL model \eqref{eq:btl}.
This setup allows us to study how rankings and ranking uncertainty vary with
observable prompt characteristics.

\paragraph{\textbf{Exclusive categories.}}
We begin by isolating the effect of individual prompt categories. As a baseline,
we consider \emph{intrinsic rankings}, corresponding to the covariate vector
$x = 0$, which capture aggregate preferences not explained by prompt categories.
We then compare these to rankings obtained under \emph{exclusive category}
inputs, where exactly one category indicator equals one.

Figure~\ref{fig:specificity_intrinsic_results} illustrates this comparison for
the Specificity category. The intrinsic rankings exhibit wide marginal confidence
intervals for most models, indicating substantial uncertainty in aggregate
preferences. One notable exception is Amazon Nova-pro, whose confidence interval
lies entirely in the lower ranks, suggesting it is consistently disfavored
regardless of prompt content.

Introducing the Specificity category alters both predicted ranks and their
uncertainty. Most models continue to exhibit wide confidence intervals,
indicating that this category alone does not sharply distinguish preferences.
However, Grok-4 stands out. Its predicted rank improves from 3 to 1, and its
confidence interval collapses from width 6 to a singleton, indicating
statistically resolved dominance for highly specific prompts. In contrast,
Qwen-Max experiences a large rank decline from 4 to 10 accompanied by a shift in its
confidence interval from $[1,8]$ to $[6,10]$.

These results highlight the importance of accounting for ranking uncertainty and prompt characteristics. While point estimates suggest a hierarchy
across models, uncertainty-aware rankings reveal that many apparent differences
are not statistically meaningful. At the same time, as Grok-4 demonstrates, our framework can identify instances where a model is preferred for a prompt.

\begin{table}[t]
\caption{\textbf{Prompt-dependent rankings and uncertainty for single-category prompts.}
Predicted rankings and 95\% marginal confidence intervals for ten LLMs under
prompts associated with a single category tag.
Each cell reports the predicted rank and 95\% marginal confidence interval when the input
covariate vector contains a single nonzero entry corresponding to the column’s
category.
Wide intervals indicate weakly identified preferences, while narrow 
intervals indicate statistically supported dominance for that category.}
\label{tab:category_tags_table}
  \begin{minipage}{0.95\columnwidth}
		\begin{center}
\resizebox{\columnwidth}{!}{
        \begin{tabular}{lccccc}
\toprule
LLM & Code & Creative Writing & Complexity & Creativity & Domain Knowledge \\
\midrule
ChatGPT-4o & 2 [1,4] & 4 [1,7] & 1 [1,5] & 3 [2,7] & 3 [1,7] \\
DeepSeek-R1 & 3 [1,6] & 1 [1,7] & 2 [1,6] & 5 [2,9] & 2 [1,6] \\
Grok-4 & 5 [3,8] & 2 [1,6] & 3 [1,8] & 1 [1,2] & 1 [1,5] \\
Qwen-Max & 1 [1,4] & 6 [1,10] & 5 [1,9] & 4 [2,9] & 6 [2,8] \\
O3 & 6 [4,9] & 7 [2,10] & 4 [1,8] & 6 [2,9] & 4 [1,7] \\
Gemma-3 & 8 [4,9] & 3 [1,7] & 6 [2,9] & 8 [4,9] & 5 [1,8] \\
Claude Sonnet-4 & 4 [1,8] & 5 [1,8] & 7 [3,10] & 7 [2,9] & 9 [7,10] \\
Mistral & 7 [4,9] & 10 [6,10] & 10 [7,10] & 2 [1,7] & 8 [5,9] \\
Llama-4 Maverick & 9 [6,10] & 8 [5,10] & 8 [3,10] & 9 [4,9] & 7 [3,9] \\
Amazon Nova & 10 [9,10] & 9 [6,10] & 9 [4,10] & 10 [10,10] & 10 [9,10] \\
\bottomrule
\end{tabular}
}
\end{center}
		\bigskip\centering
	\end{minipage}
\end{table}%

\begin{table}[t]
\caption{\textbf{Prompt-dependent rankings and uncertainty for single-category prompts
(continued).}
Predicted rankings and 95\% marginal confidence intervals for ten LLMs under
additional single-category prompt tags.
Each cell reports the predicted rank and  95\% marginal confidence interval.
The results illustrate cross-category heterogeneity in model performance.}
\label{tab:category_tags_table2}
  \begin{minipage}{\columnwidth}
		\begin{center}
\resizebox{0.95\columnwidth}{!}{
        \begin{tabular}{lccccc}
\toprule
LLM & Problem Solving & Real World & Specificity & Technical Accuracy & Math \\
\midrule
ChatGPT-4o & 3 [1,8] & 1 [1,4] & 3 [2,9] & 1 [1,6] & 3 [1,7] \\
DeepSeek-R1 & 1 [1,8] & 3 [1,6] & 2 [2,9] & 2 [1,8] & 5 [1,8] \\
Grok-4 & 9 [6,10] & 9 [4,10] & 1 [1,1] & 6 [1,10] & 2 [1,7] \\
Qwen-Max & 2 [1,8] & 2 [1,6] & 10 [6,10] & 3 [1,8] & 1 [1,7] \\
O3 & 6 [1,9] & 4 [1,9] & 4 [2,9] & 5 [1,8] & 7 [2,10] \\
Gemma-3 & 5 [1,8] & 6 [2,10] & 7 [2,10] & 8 [2,10] & 10 [7,10] \\
Claude-Sonnet-4 & 4 [1,8] & 5 [2,10] & 5 [2,9] & 7 [2,10] & 4 [1,7] \\
Mistral & 7 [1,9] & 10 [5,10] & 6 [2,10] & 4 [1,8] & 6 [1,8] \\
Llama-4 Maverick & 10 [8,10] & 8 [4,10] & 9 [2,10] & 10 [6,10] & 8 [5,10] \\
Amazon Nova & 8 [1,10] & 7 [4,10] & 8 [2,10] & 9 [6,10] & 9 [7,10] \\
\bottomrule
\end{tabular}
}
\end{center}
		\bigskip\centering

	\end{minipage}
\end{table}%

\paragraph{\textbf{Cross-category heterogeneity and task specialization.}}
Tables~\ref{tab:category_tags_table} and~\ref{tab:category_tags_table2} extend
this analysis to all ten prompt categories. Three consistent patterns emerge.

First, ChatGPT-4o and DeepSeek-R1 behave as robust generalists and are generally favorable. They rank in the
top five across all categories, and their lower confidence bounds are typically
close to the top ranks, indicating stable performance across diverse tasks.

Second, Amazon Nova-pro and Llama-4 Maverick are consistently disfavored. Across
categories, their upper confidence bounds are high, and in several cases they
are significantly ranked last, suggesting that their underperformance is robust
to prompt variation.

Third, several models exhibit clear task specialization. Grok-4 performs
particularly well on prompts associated with Creativity, Domain Knowledge, and
Specificity, often ranking first with relatively narrow confidence intervals.
Qwen-Max, by contrast, excels on Code and Math prompts but performs poorly on
creative tasks.

From a decision perspective, these patterns indicate that a single global
leaderboard is an inadequate summary of model performance. Prompt-dependent
rankings with uncertainty quantification allow a decision maker to distinguish
robust generalists from specialists and to avoid overconfident conclusions when
preferences are weakly identified.

\begin{figure}[t]
    \vskip 0.2in
  \begin{center}
  \centerline{\includegraphics[width=0.7\columnwidth]{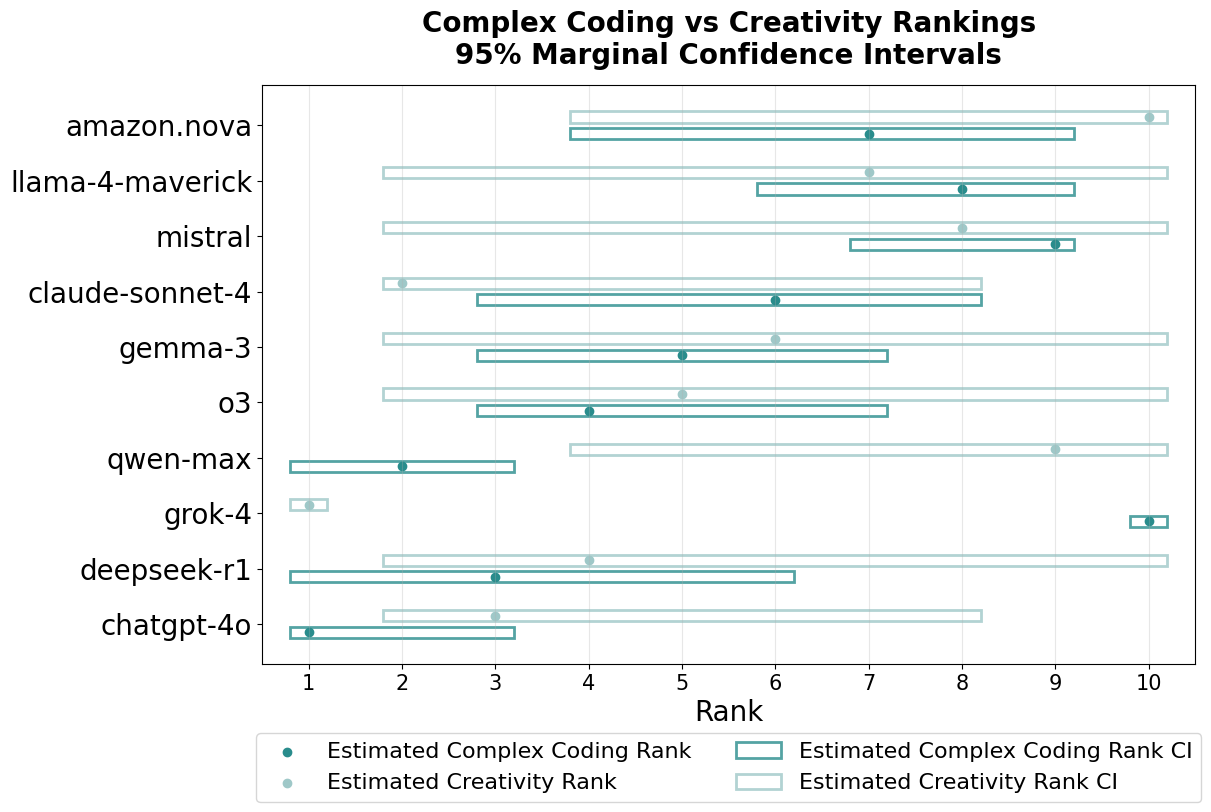}}
    \caption{\textbf{Rankings and uncertainty under multi-category prompts.}
    Predicted rankings and 95\% marginal confidence intervals for ten LLMs under
    two composite prompt types. Bold estimates correspond to prompts associated with Code, Complexity, Domain Knowledge, Problem Solving, Real World, and Technical Accuracy,
    representing complex coding tasks.  Lighter estimates correspond to prompts associated with Creativity and Creative Writing. The figure highlights task-dependent specialization and shows that ranking uncertainty differs substantially across prompt mixtures.}
    \label{fig:category_mix_comparison_results}
    \end{center}
\end{figure}

\paragraph{\textbf{Category mixtures.}}
Real-world prompts typically reflect multiple categories. To examine such
settings, we compare two composite covariate vectors. The first  covariate $x_{code}$ represents complex coding tasks, associated with Code, Complexity, Domain Knowledge, Problem
Solving, Real World, and Technical Accuracy. The second covariate  $x_{creative}$ represents creative tasks
and includes tags Creativity Creative Writing, Real World, and Specificity.

Figure~\ref{fig:category_mix_comparison_results} reports rankings and 95\% marginal confidence intervals for these two composite covariates. Confidence intervals are systematically narrower
for complex coding prompts $x_{code}$, suggesting either 
greater data availability or easier to differentiate preferences for the complicated coding task than for creative tasks. In contrast, creative prompts $x_{creative}$ exhibit wider
confidence intervals, reflecting greater subjectivity and heterogeneity in
human preferences.

The specialization pattern becomes particularly clear. Qwen-Max ranks near the
top at $2$ with a narrow confidence interval $[1, 3]$ for complex coding prompts $x_{code}$, but drops sharply in rank to $9$ for creative prompts $x_{creative}$ with confidence interval $[4, 10]$. Grok-4 displays the opposite behavior,
ranking last with a confidence interval of $[10]$ for coding tasks $x_{code}$,  but first with a singleton confidence interval $[1]$ for
creative tasks $x_{creative}$,  making it a significantly favorable model for creative tasks.

\subsection{Prompt Length}

\begin{table}[t]
	\caption{\textbf{Prompt-length–dependent ranking uncertainty.}
Predicted 95\% simultaneous confidence intervals for the ranks of five LLMs across ranges of prompt token counts.
Each entry reports the set of ranks consistent with the data at the given token length. As prompt length increases, confidence intervals widen, reflecting growing uncertainty in relative performance. Beyond approximately 1127 tokens, all models become statistically indistinguishable, and the ranking confidence sets collapse to the uninformative range $[1,5]$.}
	\label{tab:token_length_rank_ci}
	\begin{minipage}{\columnwidth}
		\begin{center}
        \resizebox{0.75\columnwidth}{!}{
			\begin{tabular}{llrlllll}
\toprule
Region Range & GPT-4-1106 & GPT-4-0613 & GPT-3.5 & Claude-2.1 & Vicuna-7b \\
\midrule
$[0,14)$      & [1] & [2,3] & [2,4] & [3,4] & [5] \\
$[14,88)$     & [1] & [2]   & [3,4] & [3,4] & [5] \\
$[89,131)$    & [1] & [2]   & [3,5] & [3,5] & [3,5] \\
$[131,240)$   & [1] & [2,3] & [3,5] & [3,5] & [2,5] \\
$[240,351)$   & [1,2] & [2,3] & [3,5] & [3,5] & [1,5] \\
$[362,470)$   & [1,2] & [2,5] & [2,5] & [2,5] & [1,5] \\
$[470,974)$   & [1,3] & [1,5] & [2,5] & [2,5] & [1,5] \\
$[979,1122)$  & [1,4] & [1,5] & [2,5] & [1,5] & [1,5] \\
$[1122,1124)$ & [1,5] & [1,5] & [1,5] & [1,5] & [1,5] \\
$[1124,1127)$ & [1,4] & [1,5] & [2,5] & [1,5] & [1,5] \\
$[1127, \infty)$ & [1,5] & [1,5] & [1,5] & [1,5] & [1,5] \\
        \bottomrule
        \end{tabular}
        }
		\end{center}
		\bigskip\centering

	\end{minipage}
\end{table}%

We next study prompt length by extending the illustrative
example in Section~\ref{sec:illegpromplength}. Table~\ref{tab:token_length_rank_ci} reports 95\% simultaneous ranking confidence intervals for five models across different
token counts. The confidence intervals were calculated for each token amount in the range $[0, 2000]$. We observe the following patterns.

For short and moderate prompts, rankings are partially identified. In
particular, for token counts between 89 and 131, GPT-4-1106 and GPT-4-0613 are significantly ranked first and second, respectively, while the remaining models are statistically indistinguishable. As prompt length increases further, confidence intervals widen, and beyond 1127 tokens all models become
indistinguishable.

This behavior reflects increasing uncertainty in relative performance for long prompts. Proposition~\ref{prop:asymptotic_extrapolation} characterizes this phenomenon asymptotically. Along the extrapolation direction corresponding to increasing prompt length, the estimated covariate effects governing relative performance are not statistically resolved. Specifically, we take the direction $v = \mathbf 1$ and interpret $\lambda$ as the number of tokens in a prompt. In this setting, the covariate parameters $\hat{\beta}_i$ are scalar, and the limiting normalized confidence sets for utility differences take the form
\[
    C_\infty(1-\alpha, S, 1)
    =
    \prod_{(i,j)\in S}
    \left[
    \hat{\beta}_i - \hat{\beta}_j
    \pm
    t^{\infty}_{\mathrm{symm},\alpha/2}
    \sqrt{\widehat{\mathrm{Var}}(\hat{\beta}_i-\hat{\beta}_j)}
    \right].
\]
Because these differences are not statistically distinguishable from zero for
any pair of models, all pairwise orderings remain unresolved in the limit.
Consequently, the induced ranking confidence sets converge to the uninformative
range $[1,5]$ for all LLMs, as illustrated in
Figure~\ref{fig:token_extrapolation_uncertainty_results} in the Appendix.

\subsection{Decision Implications}

The experimental results have direct implications for ranking-based decision
making. First, they show that acting on point-estimate rankings can lead to
overconfident and potentially inefficient decisions, especially when preference
signals are weak or heterogeneous. In many settings, uncertainty-aware rankings
reveal partial identification rather than strict dominance, suggesting that
robust decision rules should avoid committing to a single ordering.

Second, prompt-dependent rankings enable specialization-aware decisions. Rather
than deploying a single globally ranked model, a decision maker can route
prompts to models that are statistically supported as dominant for the relevant
prompt characteristics. This is particularly important in mixed or heterogeneous
environments, where global rankings obscure task-specific strengths.

Finally, when rankings become uninformative, as in the case of long prompts,
uncertainty-aware inference provides a principled signal to refrain from
fine-grained selection. In such cases, decisions based on cost, latency, or
robustness considerations may be more appropriate than ranking-based selection.
Overall, these results highlight the importance of treating rankings as
decision-relevant objects and incorporating uncertainty directly into
ranking-based mechanisms.

\section{Related Work}
\label{sec:relatedwork}
This paper relates to three strands of literature: ranking LLMs using human preference data, statistical inference for rankings, and contextual ranking models with covariates.

\paragraph{\textbf{Ranking LLMs from human preferences.}}
A growing literature evaluates and ranks large language models using pairwise human preference data, exemplified by systems such as LMArena and related benchmarks \citep{bai2022traininghelpfulharmlessassistant, wang2023aligning, chiang2024chatbot, carraro2025enhancing}. These approaches typically model comparisons using a BTL model with a single latent utility per model, estimated via maximum likelihood. The resulting point estimates are reported as leaderboards and increasingly used to guide deployment, routing, and model selection decisions.

While effective for large-scale aggregation, this approach implicitly assumes that model quality is homogeneous across inputs and that estimated rankings are well identified. In practice, LLM performance varies substantially across prompts and task types, and rankings are derived from noisy and finite samples of human judgments. Uncertainty is usually quantified at the utility level, if at all, and ranks are inferred indirectly. As a result, many reported rank differences may not be statistically meaningful, yet are still acted upon in downstream systems. Our work addresses these limitations by allowing rankings to depend explicitly on the evaluation context and by providing valid uncertainty quantification directly at the level of rankings.

\paragraph{\textbf{Statistical inference for ranks.}}
A second strand studies statistical inference for rankings more broadly.
\citet{hall2009using} propose a bootstrap approach to approximate the distribution of predicted ranks. Other methods proceed by constructing confidence intervals for individual utilities and inferring ranks as a secondary step \citep{klein2020joint,gao2023uncertainty}. Several recent works instead focus on confidence sets constructed from pairwise utility differences \citep{liu2023lagrangian,fan2024covariates,fan2025ranking}, which more directly characterize the objects that determine rankings.
A key contribution in this literature is \citet{mogstad2024inference}, who show that rank inference based on marginal confidence intervals for individual utilities can lead to excessively wide confidence sets and coverage strictly exceeding the nominal level. They also demonstrate that bootstrap-based rank inference can fail to satisfy pointwise coverage requirements in certain settings. That work provides general and valid constructions for rank confidence sets under minimal assumptions.

Our approach builds directly on these insights. We adopt confidence sets based on pairwise utility differences and extend the framework to contextual pairwise comparison models, where the objects of interest are rankings conditional on observed covariates. This extension introduces new challenges, as ranks become functions of both estimated parameters and the evaluation context.

\paragraph{\textbf{Contextual and covariate-based ranking models.}}
A third strand incorporates covariates into ranking models, allowing utilities to depend on observable attributes. \citet{fan2024covariates} study a BTL model with item-level covariates and estimate latent utility parameters via maximum likelihood. \citet{Han08012026} allow utility parameters to evolve over time while holding item covariates fixed, and \citet{dong2025statistical} allow covariates to vary across comparisons, with an application to ranking tennis players over their careers. \citet{chau2022spectral} develop a related framework using spectral methods for latent item-covariate utility functions.
These works focus primarily on estimation and prediction of utilities and rankings but do not provide uncertainty quantification for rankings themselves, nor do they treat the evaluation context as the primary driver of ranking variation. 

In contrast, our setting treats the prompt as the key covariate, reflecting the fact that LLM performance is inherently task- and query-dependent. Our contribution is to combine contextual ranking models with valid statistical inference for ranks, enabling prompt-dependent rankings that can be used safely in downstream decision-making.

\section{Conclusion and Discussion}
\label{sec:conclusion}
This paper studies ranking as a decision-relevant inference problem in the evaluation and deployment of LLMs. LLMs are increasingly used as core components of economic and computational systems, and rankings derived from human preference data play a central role in guiding model selection, routing, and resource allocation. We introduce a statistical framework for prompt-dependent ranking inference that accounts for both contextual heterogeneity in LLM performance and uncertainty arising from finite and noisy human evaluations. By considering the contextual BTL model to allow prompt-dependent utility functions, we provide a principled way to construct rankings that vary with the input prompt and come with valid coverage guarantees. Our empirical analysis using large-scale human preference data demonstrates that LLM rankings are highly sensitive to prompt characteristics such as length and semantic category. Many apparent rank differences suggested by point-estimate leaderboards are not statistically distinguishable once uncertainty is accounted for. At the same time, our framework identifies cases of statistically supported dominance, revealing meaningful task specialization across models. These results highlight a fundamental limitation of global LLM leaderboards: when rankings are treated as fixed and well identified, they can induce overconfident and potentially inefficient deployment decisions.

From an economic perspective, our findings underscore that LLM rankings should be viewed as inputs to decision-making mechanisms rather than as definitive performance summaries. Uncertainty-aware, prompt-dependent rankings enable robust allocation rules that exploit dominance when it is supported by the data and otherwise avoid committing to arbitrary orderings. This perspective is particularly important for modern LLM systems, where queries are heterogeneous and models differ along multiple dimensions such as accuracy, cost, latency, and specialization.

Several directions for future research follow naturally from this work. One promising avenue is mechanism design for adaptive evaluation, in which human comparisons are selected strategically to reduce ranking uncertainty with minimal evaluation cost. Another direction is to incorporate ranking uncertainty directly into LLM routing and selection problems under budget or latency constraints. The framework can also be extended to account for user-specific preferences, strategic feedback, or interactions among models in ensemble systems. On the theoretical side, it is of interest to study asymptotic regimes in which both the number of LLMs and the number of comparisons grow, as well as to develop nonparametric or robust variants that maintain valid rank inference under model misspecification. More broadly, this work points toward a unified view of LLM evaluation in which statistical inference, uncertainty, and economic decision-making are treated jointly, reflecting the increasingly central role that LLMs play in modern algorithmic and market-based systems.

\bibliography{main}



\newpage
\appendix
\begin{center}
    \Large{Appendix to ``Prompt-Dependent Ranking of Large Language Models with Uncertainty Quantification"}
\end{center}

\vspace{0.1in}
\noindent
This appendix provides additional numerical results in Section~\ref{sec:appendilluseg}, as well as proofs of the theoretical results stated in the main paper, collected in Section~\ref{sec:proofs}.

\vspace{0.1in}

\section{Additional Numerical Results}
\label{sec:appendilluseg}

Table~\ref{tab:token_regions_ranking_table} reports the prompt-dependent
point-estimate rankings corresponding to Figure~\ref{fig:token_length_bump_chart}. As the token count increases, the estimated ordering changes multiple times, reflecting intersections of the fitted utility functions. In particular, Vicuna-7b is predicted to achieve the top rank for sufficiently long prompts, even beyond its nominal context window. This behavior illustrates a limitation of relying solely on point-estimate rankings and motivates the need for explicit uncertainty quantification.

\begin{table}[ht]
\caption{Predicted preference rankings of LLMs across token counts in a prompt. Vicuna-7b is predicted to be ranked first for high token count prompts, even though the prompts may be longer than its context window. This shows a limitation in the predictions and why it is important to quantify the uncertainty of the rankings.}
\label{tab:token_regions_ranking_table}
  \begin{minipage}{\columnwidth}
		\begin{center}
        \begin{tabular}{c c c c c c}
\toprule
Region Range & Rank 1 & Rank 2 & Rank 3 & Rank 4 & Rank 5 \\
\midrule
$[0, 159]$    &gpt-4-1106  &gpt-4-0613    &gpt-3.5 &claude-2.1   &vicuna-7b\\
$(159, 163]$  &gpt-4-1106    &gpt-4-0613  &gpt-3.5   &vicuna-7b  &claude-2.1\\
$(163, 265]$  &gpt-4-1106    &gpt-4-0613  &vicuna-7b &gpt-3.5  &claude-2.1\\
$(265, 267]$  &gpt-4-1106    &vicuna-7b   &gpt-4-0613 &gpt-3.5  &claude-2.1\\
$(267, 455]$  &gpt-4-1106    &vicuna-7b   &gpt-4-0613 &claude-2.1   &gpt-3.5\\
$(455, 1212]$ &vicuna-7b    &gpt-4-1106  &gpt-4-0613    &claude-2.1  &gpt-3.5\\
$(1212, 9131]$    &vicuna-7b   &gpt-4-0613 &gpt-4-1106   &claude-2.1 &gpt-3.5\\
$(9131, \infty)$    &vicuna-7b   &gpt-4-0613 &claude-2.1   &gpt-4-1106 &gpt-3.5\\
\bottomrule
\end{tabular}
\end{center}
		\bigskip\centering

	\end{minipage}
\end{table}%

To complement these results, Figure~\ref{fig:token_extrapolation_uncertainty_results}
visualizes the 95\% simultaneous confidence intervals for utility differences
computed using only the covariate (slope) parameters. These intervals correspond to the asymptotic extrapolation regime analyzed in Proposition~\ref{prop:asymptotic_extrapolation}, in which the prompt length grows unboundedly and intrinsic utility components become negligible. As shown in the figure, all pairwise utility-difference intervals contain zero, indicating that no relative ordering is statistically resolved in this regime. That is, the induced ranking confidence sets span the full range of possible ranks, which suggests that the rankings are uninformative.

\begin{figure}[t]
  \begin{center}
  \centerline{\includegraphics[width=\columnwidth]{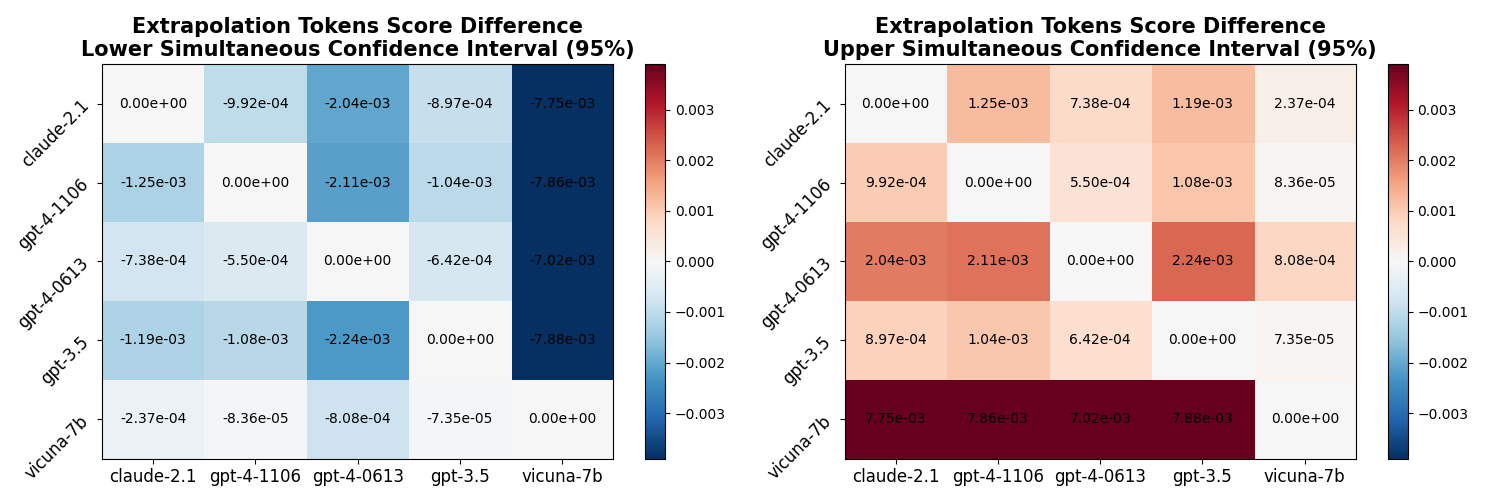}}
    \caption{Asymptotic 95\% simultaneous confidence intervals for covariate-only utility differences under prompt-length extrapolation. The left panel shows lower bounds and the right panel shows upper bounds. All intervals contain zero, indicating that no pairwise ordering is statistically resolved in the extreme prompt-length regime.}
    \label{fig:token_extrapolation_uncertainty_results}
    \end{center}
\end{figure}

\section{Proofs}
\label{sec:proofs}
In this section, we provide the proofs of the main theoretical results. These depend on the gradient and Hessian of the loss function, provided below
\begin{align*}
   \nabla \ell(\tilde{\beta}) &= \sum_{l=1}^{L}-\left(y_{a_l} -  \sigma(\tilde{\beta}^\top\tilde{x}_{a_l})\right)\tilde{x}_{a_l} \\
   \nabla^2 \ell(\tilde{\beta}) &= \sum_{l=1}^{L} \sigma(\tilde{\beta}^\top\tilde{x}_{a_l})(1-\sigma(\tilde{\beta}^\top\tilde{x}_{a_l}))\tilde{x}_{a_l}\tilde{x}_{a_l}^\top.\\
\end{align*}

\subsection{Proof of Theorem \ref{thm:mle_asymptotic_normality}}
\subsubsection{Main Proof of Theorem \ref{thm:mle_asymptotic_normality}}
\begin{proof}
    By the Mean Theorem for the utility function
    \begin{equation*}
        \nabla\ell(\hat{\beta}) = \nabla\ell(\tilde{\beta}^*) + \nabla^2\ell(\bar{\beta})(\hat{\beta} - \tilde{\beta}^*)
    \end{equation*}
    for $\bar{\beta} = \tilde{\beta}^* + \tau(\hat{\beta} - \tilde{\beta}^*)$ for some $\tau\in[0,1]$. Since $\hat{\beta}, \tilde{\beta}^*\in\Theta$, $\mathcal{P}(\hat{\beta} - \tilde{\beta}^*) = (\hat{\beta} - \tilde{\beta}^*)$. We can rewrite the above as
    \begin{equation*}
        \mathcal{P}\nabla\ell(\hat{\beta}) = \mathcal{P}\nabla\ell(\tilde{\beta}^*) + \mathcal{P}\nabla^2\ell(\bar{\beta})\mathcal{P}(\hat{\beta} - \tilde{\beta}^*).
    \end{equation*}
    Since $\mathcal{P}\nabla\ell(\hat{\beta}) = 0$, by rescaling we can write
    \begin{equation*}
        \mathcal{P}\frac{1}{\sqrt{L}}\nabla\ell(\tilde{\beta}^*) = -\mathcal{P}\left(\frac{1}{L}\nabla^2\ell(\bar{\beta})\right)\mathcal{P}\sqrt{L}(\hat{\beta} - \tilde{\beta}^*).
    \end{equation*}
    From Lemma \ref{lem:bar_hessian_convergence}, we know $\frac{1}{L}\nabla^2\ell(\bar{\beta}) - \frac{1}{L}\mathcal{I}_L(\tilde{\beta}^*) \xrightarrow{P} 0$. Therefore, using Slutsky's Theorem and Lemma \ref{lem:score_asymptotic_distribution}, we can write 
     \begin{align*}
         \mathcal{P}\frac{1}{\sqrt{L}}\nabla\ell(\tilde{\beta}^*) &\approx -\mathcal{P}\bar{\mathcal{I}}(\tilde{\beta}^*)\mathcal{P}\sqrt{L}(\hat{\beta} - \tilde{\beta}^*)\\
         \rightarrow \sqrt{L}(\hat{\beta} - \tilde{\beta}^*) &\approx -(\mathcal{P}\bar{\mathcal{I}}(\tilde{\beta}^*)\mathcal{P})^\dagger\mathcal{P}\frac{1}{\sqrt{L}}\nabla\ell(\tilde{\beta}^*),
     \end{align*}
     where $A^\dagger$ is the Moore-Penrose Pseudoinverse. By Lemma \ref{lem:score_asymptotic_distribution}, \begin{equation*}
         (\mathcal{P}\bar{\mathcal{I}}(\tilde{\beta}^*)\mathcal{P})^\dagger\mathcal{P}\frac{1}{\sqrt{L}}\nabla\ell(\tilde{\beta}^*) \xrightarrow{d} \mathcal{N}(0, \Sigma),
     \end{equation*} 
     where $\Sigma = (\mathcal{P}\bar{\mathcal{I}}\mathcal{P})^\dagger\mathcal{P}\bar{\mathcal{I}}\mathcal{P}(\mathcal{P}\bar{\mathcal{I}}\mathcal{P})^\dagger = (\mathcal{P}\bar{\mathcal{I}}\mathcal{P})^\dagger$. The result follows. 
\end{proof}

\subsubsection{Auxiliary Lemmas for the Proof of Theorem \ref{thm:mle_asymptotic_normality}}
\begin{lemma}
    Suppose Assumptions~\ref{as:comparison_graph_connectivity}--\ref{as:limit_fisher} hold. Then
    \begin{equation*}
        \hat{\beta}_L \xrightarrow{P} \tilde{\beta}^*
    \end{equation*}
    \label{lem:consistency}
\end{lemma} 

\begin{proof}[Proof of Lemma \ref{lem:consistency}]
    We want to show that for any $\epsilon > 0$, the probability that $\hat{\beta}$ lies within the ball $\mathcal{B}_\epsilon(\tilde{\beta}^*) = \{\tilde{\beta}: \lVert \tilde{\beta}-\tilde{\beta}^*\rVert \leq \epsilon\}$ approaches 1. It suffices to show that for every $\epsilon > 0$
    \begin{equation*}
        \lim_{L\rightarrow \infty}\mathbb{P}(\inf_{\tilde{\beta} \in\partial\mathcal{B}_\epsilon(\tilde{\beta}^*)} \ell(\tilde{\beta}) > \ell(\tilde{\beta}^*) ) = 1.
    \end{equation*}
    Consider the Taylor expansion around $\tilde{\beta}^*$
    \begin{equation*}
        \ell(\tilde{\beta}) = \ell(\tilde{\beta}^*) + (\tilde{\beta}-\tilde{\beta}^*)^\top\nabla \ell(\tilde{\beta}^*) + \frac{1}{2}(\tilde{\beta}-\tilde{\beta}^*)^\top\nabla^2\ell(\bar{\beta})(\tilde{\beta}-\tilde{\beta}^*)
    \end{equation*}
    for some $\tau\in[0,1]$, $\bar{\beta} = \tilde{\beta} + \tau(\tilde{\beta}^*-\tilde{\beta})$. Let $\Delta \ell(\tilde{\beta}) = \ell(\tilde{\beta}) -\ell(\tilde{\beta}^*)$ and $u = \tilde{\beta}-\tilde{\beta}^*$. We focus on the case where $\Vert u\rVert = \epsilon$. Then
    \begin{equation*}
       \ell_L(\tilde{\beta}) - \ell_L(\tilde{\beta}^*) = \underbrace{\nabla \ell_L(\tilde{\beta}^*)^\top u}_{t_1} + \underbrace{\frac{1}{2} u^\top \nabla^2 \ell_L(\tilde{\beta}^*) u}_{t_2} + \underbrace{\frac{1}{2} u^\top \left(\nabla^2 \ell_L(\bar{\beta}) - \nabla^2 \ell_L(\tilde{\beta}^*)\right) u}_{t_3}
    \end{equation*}
    For $t_1$, by Lemma~\ref{lem:score_asymptotic_distribution}, $\frac{1}{\sqrt{L}}\nabla\ell(\tilde{\beta}^*) \xrightarrow{d} \mathcal{N}(0, \bar{\mathcal{I}}(\tilde{\beta}^*))$. Therefore
    \begin{equation*}
        \lVert \nabla \ell(\tilde{\beta}^*)\rVert_2 = O_p\left(\sqrt{L}\right)
    \end{equation*} and 
    \begin{equation*}
        u^\top\nabla \ell(\tilde{\beta}^*) = O_p\left(\sqrt{L}\right)\lVert u\rVert = O_p\left(\sqrt{L}\right)\epsilon
    \end{equation*}
    For the $t_2$, 
    By Lemma~\ref{lem:hessian_wlln}, we have $\frac{1}{L}\nabla^2\ell(\tilde{\beta}^*) \xrightarrow{P} \bar{\mathcal{I}}$. Since $\bar{\mathcal{I}}$ is positive definite, we can bound this term from below with the minimum eigenvalue $\lambda_{min}(\bar{\mathcal{I}}) > 0$. Therefore, with probability approaching 1,
    \begin{equation*}
    \frac{1}{2}u^\top\nabla^2\ell(\tilde{\beta}^*)u \geq \frac{1}{2} L \lambda_{min}(\bar{\mathcal{I}})\lVert u\rVert^2 = K\epsilon^2 L.
     \end{equation*}
     For $t_3$, we observe that the third derivatives of $\ell(\tilde{\beta})$ are bounded by our assumptions of the covariates. This implies that $\nabla^2 \ell(\tilde{\beta})$ is Lipschitz continuous on $\Theta$ and
    \begin{equation*}
        \left\lVert \nabla^2 \ell(\bar{\beta}) - \nabla^2 \ell(\tilde{\beta}^*)\right\rVert \leq CL\lVert \bar{\beta} - \tilde{\beta}^*\rVert \leq CL\epsilon.
    \end{equation*}
    Then, the remainder is bounded by
    \begin{equation*}
    \frac{1}{2} \lVert u \rVert^2 \lVert \nabla^2 \ell_L(\bar{\beta}) - \nabla^2 \ell_L(\tilde{\beta}^*) \rVert_{op} \leq \frac{1}{2} \epsilon^2 (C L \epsilon) = O(L \epsilon^3).
    \end{equation*}
Combining these results, with probability approaching 1,
     \begin{equation*}
         \ell(\tilde{\beta}) -\ell(\tilde{\beta}^*) \geq - |O_p(\sqrt{L})| \epsilon + KL\lambda_{\min} \epsilon^2 - C L\epsilon^3.
     \end{equation*}
We choose $\epsilon$ small enough such that the quadratic term dominates the cubic term, then with probability approaching 1, $\ell(\tilde{\beta}) > \ell(\tilde{\beta}^*)$ for all $\tilde{\beta} \in \partial\mathcal{B}_\epsilon(\tilde{\beta}^*)$. Since $\ell(\tilde{\beta})$ is a continuous function of $\tilde{\beta}$, this implies there exists a local minimum $\hat{\beta}_L$ inside the ball. Therefore
     \begin{equation*}
         \lim_{L\rightarrow\infty} P(\lVert \hat{\beta} - \tilde{\beta}^*\rVert \leq \epsilon) = 1.
     \end{equation*}
\end{proof}

\begin{lemma}
    Suppose $\lVert \tilde{x}_{a_l}\rVert_2 \leq B_2$ and $|\tilde{\beta}^{\top}\tilde{x}_{a_l}| < \infty$ for all $l\in[L]$ and $a\in\mathcal{E}$. Then
    \begin{equation*}
        \frac{1}{L}\nabla^2 \ell(\tilde{\beta}) \xrightarrow{P} \frac{1}{L}\mathcal{I}_L(\tilde{\beta}).
    \end{equation*}
    Furthermore, if $\lim_{L\rightarrow\infty} \frac{1}{L}\mathcal{I}_L(\tilde{\beta}) = \bar{\mathcal{I}}(\tilde{\beta})$, then
    \begin{equation*}
        \frac{1}{L}\nabla^2 \ell(\tilde{\beta}) \xrightarrow{P} \bar{\mathcal{I}}(\tilde{\beta}).
    \end{equation*}
    \label{lem:hessian_wlln}
\end{lemma}
\begin{proof}[Proof of Lemma \ref{lem:hessian_wlln}]
    Let $H^{(l)} = \sigma(\tilde{\beta}^\top \tilde{x}_{a_l})(1-\sigma(\tilde{\beta}^\top \tilde{x}_{a_l}))\tilde{x}_{a_l}\tilde{x}_{a_l}^\top$ such that $\nabla^2\ell(\tilde{\beta}) = \sum_{l=1}^L H^{(l)}$. Since $\lVert \tilde{x}_{a_l}\rVert_2 \leq B_2$ and $|\tilde{\beta}^{\top}\tilde{x}_{a_l}| < \infty$, we have $H^{(l)}_{u,v} \leq c_1$ for some $c_1 > 0$ for all matrix elements $u,v$. Since these values are bounded, their variance $\text{Var}(H^{(l)}_{u,v}) \leq C_1$ for some $C_1 > 0$. By Chebyshev's Inequality, for any $\epsilon > 0$ and all elements $u,v$
    \begin{equation*}
        P\left(\left|\frac{1}{L}\sum_{l=1}^L H_{uv}^{(l)} - \mathbb{E}[H_{uv}^{(l)}]\right| > \epsilon\right) \leq \frac{\sum_{l=1}^L \text{Var}(H_{uv}^{(l)})}{L^2\epsilon^2} \leq \frac{LC_1}{L^2\epsilon^2} = \frac{C_1}{L\epsilon^2}.
    \end{equation*}
    As $L\rightarrow\infty$, this probability converges to 0. Since \begin{equation*}
        \frac{1}{L}\sum_{l=1}^L \mathbb{E}[H^{(l)}] = \frac{1}{L}\sum_{l=1}^L\sum_{a\in\mathcal{E}}p_{a}\sigma(\tilde{\beta}^\top \tilde{x}_{a_l})(1-\sigma(\tilde{\beta}^\top \tilde{x}_{a_l}))\tilde{x}_{a_l}\tilde{x}_{a_l}^\top = \frac{1}{L}\mathcal{I}_L(\tilde{\beta}),
    \end{equation*}
    the first result follows. \\
    If $\lim_{L\rightarrow\infty} \frac{1}{L}\mathcal{I}_L(\tilde{\beta}) = \bar{\mathcal{I}}(\tilde{\beta})$, then
    \begin{equation*}
        \lVert\frac{1}{L}\nabla^2 \ell(\tilde{\beta}) - \bar{\mathcal{I}}(\tilde{\beta})\rVert \leq \lVert\frac{1}{L}\nabla^2 \ell(\tilde{\beta}) - \frac{1}{L}\mathcal{I}_L(\tilde{\beta})\rVert + \lVert\frac{1}{L}\mathcal{I}_L(\tilde{\beta}) - \bar{\mathcal{I}}(\tilde{\beta})\rVert \xrightarrow{P} 0
    \end{equation*}
\end{proof}

\begin{lemma}
    Let $\bar{\beta} = \tilde{\beta}^* + \tau(\hat{\beta} - \tilde{\beta}^*)$ for some $\tau\in[0,1]$. Suppose $\lVert \tilde{x}_{a_l}\rVert_2 \leq B_2$ and $|\tilde{\beta}^{\top}\tilde{x}_{a_l}| < \infty$ for all $l\in[L]$ and $a\in\mathcal{E}$. Then
    \begin{equation*}
        \frac{1}{L}\nabla^2 \ell(\bar{\beta}) \xrightarrow{P} \frac{1}{L}\mathcal{I}_L(\tilde{\beta}^*).
    \end{equation*}
    \label{lem:bar_hessian_convergence}
\end{lemma}
\begin{proof}[Proof of Lemma \ref{lem:bar_hessian_convergence}]
    \begin{equation*}
        \left\lVert \frac{1}{L}\nabla^2 \ell(\bar{\beta}) - \frac{1}{L}\mathcal{I}_L(\tilde{\beta}^*)\right\rVert_\infty 
        \leq 
        \underbrace{\left\lVert \frac{1}{L}\nabla^2 \ell(\bar{\beta}) - \frac{1}{L}\nabla^2 \ell(\tilde{\beta}^*)\right\rVert_\infty}_{t_1} + \underbrace{\left\lVert \frac{1}{L}\nabla^2 \ell(\tilde{\beta}^*) - \frac{1}{L}\mathcal{I}_L(\tilde{\beta}^*)\right\rVert_\infty}_{t_2}
    \end{equation*}
    By Lemma \ref{lem:hessian_wlln}, $t_2 = o_p(1)$. For $t_1$, we observe that the third derivatives of $\ell(\tilde{\beta})$ are bounded by our assumptions of the utilities and covariates. This implies that $\frac{1}{L}\nabla^2 \ell(\tilde{\beta})$ is Lipschitz continuous on $\Theta$ and
    \begin{equation*}
        \left\lVert \frac{1}{L}\nabla^2 \ell(\bar{\beta}) - \frac{1}{L}\nabla^2 \ell(\tilde{\beta}^*)\right\rVert_\infty \leq K\lVert \bar{\beta} - \tilde{\beta}^*\rVert_\infty.
    \end{equation*}
    Since $\bar{\beta}$ is between $\hat{\beta}$ and $\tilde{\beta}^*$, the consistency of $\hat{\beta}$ shown in Lemma~\ref{lem:consistency} implies $\lVert \bar{\beta} - \tilde{\beta}^*\rVert_\infty = o_p(1)$. The result follows from this. 
\end{proof}

\begin{lemma}
    Suppose $\bar{\mathcal{I}}(\tilde{\beta}) = \lim_{L\rightarrow\infty}\frac{1}{L}\mathcal{I_L}(\tilde{\beta})$ exists for some sequence of inputs $\{x_l\}_{l=1}^L$. Then
    \begin{equation*}
        \frac{1}{\sqrt{L}}\nabla\ell(\tilde{\beta}^*) \xrightarrow{d} \mathcal{N}(0, \bar{\mathcal{I}}(\tilde{\beta}^*)).
    \end{equation*}
    \label{lem:score_asymptotic_distribution}
\end{lemma}
\begin{proof}[Proof of Lemma \ref{lem:score_asymptotic_distribution}]
    Let $s_l = -\left(y_{a_l} -  \sigma(\tilde{\beta}^\top \tilde{x}_{a_l})\right)\tilde{x}_{a_l}$ such that $\nabla\ell(\tilde{\beta}^*) = \sum_{l=1}^L s_l$. Note that $\mathbb{E}[s_l] = 0$ and
    \begin{equation*}
        \text{Var}(s_l) = \sum_{a\in\mathcal{E}}p_{a} \sigma(\tilde{\beta}^\top \tilde{x}_{a_l})(1-\sigma(\tilde{\beta}^\top \tilde{x}_{a_l}))\tilde{x}_{a_l}\tilde{x}_{a_l}^\top.
    \end{equation*}
    Since $\lVert \tilde{x}_{a_l}\rVert_2 \leq B_2$ and $|\tilde{\beta}^{\top}\tilde{x}_{a_l}| < \infty$, we have that $\lVert s_l \rVert < c_2$. Therefore, for any $\epsilon > 0$, there exists $L_0$ large enough such that $\epsilon\sqrt{L} > c_2$ for all $L \geq L_0$ and
    \begin{equation*}
        \lim_{L \to \infty} \frac{1}{L} \sum_{l=1}^L E\left[ \|s_l\|^2 \cdot \mathds{1}(\|s_l\| > \epsilon \sqrt{L}) \right] = 0.
    \end{equation*}
    By the Lindeberg-Feller Central Limit Theorem,
    \begin{equation*}
        \frac{1}{\sqrt{L}}\nabla\ell(\tilde{\beta}^*) \xrightarrow{d} \mathcal{N}(0, \bar{\mathcal{I}}(\tilde{\beta}^*)).
    \end{equation*}
\end{proof}

\subsection{Proof of Corollary \ref{cor:score_difference_asymptotic_normality}}
\begin{proof}
    Given $\tilde{X}_S$, let $\Delta_S(\tilde{\beta}) = \tilde{X}_S\tilde{\beta}$. This is a continuous and differentiable function of $\tilde{\beta}$ with Jacobian $\tilde{X}_S$. By Theorem~\ref{thm:mle_asymptotic_normality}
\begin{equation*}
    \sqrt{L}(\hat{\beta}-\tilde{\beta}^*) \xrightarrow{d} \mathcal{N}(0, (\mathcal{P}\bar{\mathcal{I}}\mathcal{P})^\dagger).
\end{equation*}
From the Delta Method \cite{Vaart_1998}, it follows that
\begin{equation*}
    \sqrt{L}(\Delta_S(\hat{\beta})-\Delta(\tilde{\beta}^*)) \xrightarrow{d} \mathcal{N}(0, \tilde{X}_S(\mathcal{P}\bar{\mathcal{I}}\mathcal{P})^\dagger\tilde{X}_S^\top).
\end{equation*}
\end{proof}

\subsection{Proof of Theorem \ref{thm:score_differences_ci_coverage}}
\begin{proof}
    In this section, we establish the asymptotic coverage of the simultaneous confidence intervals for utility differences. We focus on the symmetric intervals $C_{\mathrm{symm},L}$ for notation since the logic for one-sided intervals follows similarly.

First, we establish the consistency of the bootstrap covariance estimator. Let $\hat{\beta}^*$ denote the bootstrap estimator. Under Assumption~\ref{as:boundedness} (bounded covariates) and Assumption~\ref{as:comparison_graph_connectivity} (connected comparison graph), the logistic log-likelihood is strictly convex and smooth with high probability. Following Theorem 2 of \citet{mammen2012does}, these regularity conditions imply that the bootstrap distribution of $\sqrt{L}(\hat{\beta}^* - \hat{\beta})$ consistently estimates the limiting distribution of the MLE. Furthermore, these conditions ensure the uniform integrability of the squared bootstrap error, implying the convergence of the bootstrap covariance matrix to the true asymptotic covariance \cite{shao1995jackknife}
\begin{equation}
    L\hat{\Sigma} \xrightarrow{P} \Sigma_{\infty} \equiv (\mathcal{P}\bar{\mathcal{I}}\mathcal{P})^\dagger.
    \label{eq:cov_consistency}
\end{equation}
Consequently, the estimated standard error for any pair $(i,j)$ is consistent:
\begin{equation}
    \frac{\hat{se}_{ij}(x)}{se_{ij}(x)} \xrightarrow{P} 1,
    \label{eq:se_consistency}
\end{equation}
where $se_{ij}(x) = \sqrt{\frac{1}{L}\tilde{x}_{ij}^\top\Sigma_{\infty}\tilde{x}_{ij}}$ is the true asymptotic standard error.

Next, we derive the limiting distribution of the test statistic. Let $Z$ be the random vector of standardized differences for $(i,j) \in S$, with elements:
\begin{equation*}
    Z_{ij} = \frac{\tilde{x}_{ij}^\top(\hat{\beta}-\tilde{\beta}^*)}{\hat{se}_{ij}(x)}.
\end{equation*}
By Corollary~\ref{cor:score_difference_asymptotic_normality}, the vector of unscaled errors $\sqrt{L}\tilde{x}_{ij}^\top(\hat{\beta}-\tilde{\beta}^*)$ converges in distribution to a multivariate Gaussian with covariance matrix determined by $\Sigma_{\infty}$. Combining this with the consistency of the standard errors in \eqref{eq:se_consistency}, it follows by Slutsky's Theorem that $Z$ converges to a standardized multivariate Gaussian:
\begin{equation*}
    Z \xrightarrow{d} \mathcal{Z} \sim \mathcal{N}(0, R),
\end{equation*}
where $R$ is the limiting correlation matrix with entries
\begin{equation*}
    R_{(i,j), (k,l)} = \frac{\tilde{x}_{ij}^\top \Sigma_{\infty} \tilde{x}_{kl}}{\sqrt{(\tilde{x}_{ij}^\top \Sigma_{\infty} \tilde{x}_{ij})(\tilde{x}_{kl}^\top \Sigma_{\infty} \tilde{x}_{kl})}}.
\end{equation*}
Since the max-statistic $\mathcal{T}_{\mathrm{symm}} = \max_{(i,j)\in S} |Z_{ij}|$ is a continuous functional of $Z$, the Continuous Mapping Theorem implies
\begin{equation*}
    \mathcal{T}_{\mathrm{symm}} \xrightarrow{d} \mathcal{T}_{\infty} \equiv \max_{(i,j)\in S} |\mathcal{Z}_{ij}|.
\end{equation*}

Finally, we address the convergence of the critical values. The estimated critical value $\hat{t}_{\mathrm{symm}}(\alpha, x)$ is the $(1-\alpha)$-quantile of the max-statistic simulated from $\mathcal{N}(0, \hat{\Sigma})$. As shown in \eqref{eq:cov_consistency}, the covariance estimate $\hat{\Sigma}$ is consistent, which implies the estimated correlation matrix $\hat{R}$ converges in probability to $R$. Following 
\citet{hothorn2008simultaneous}, the distribution of the test statistic calculated with the estimated correlation matrix converges to the true limit distribution. Thus, the simultaneous confidence intervals maintain the coverage rate asymptotically
\begin{equation*}
\lim_{L\rightarrow\infty}P\left(\Delta_S(x) \in C(1-\alpha, S, x)\right) = \lim_{L\rightarrow\infty}P\left(\max_{(i,j)\in S} \left| \frac{\tilde{x}_{ij}^\top(\hat{\beta} - \tilde{\beta}^*)}{\hat{se}_{ij}(x)} \right| \leq \hat{t}_{symm}(\alpha,x) \right)
\geq 1-\alpha
\end{equation*}
\end{proof}

\subsection{Proof of Theorems~\ref{thm:rank_marginal_ci_coverage} and \ref{thm:rank_simultaneous_ci_coverage}}
The proofs of the validity of the ranks confidence sets
follow the same logical structure as Theorems~3.1 and~3.3 in
\citet{mogstad2024inference}, with the key distinction that utility differences and confidence sets are indexed by the covariate $x$. 

\subsection{Proof of  Proposition \ref{prop:asymptotic_extrapolation}}
In this section, we explore the behavior of the predicted ranks and confidence sets as the input covariates grow unboundedly in a fixed direction, as stated in Proposition \ref{prop:asymptotic_extrapolation}. More formally, Let the estimated utilities be defined as $\hat{\theta}(x) = \hat{\beta}_{0i} + x^\top\hat{\beta}_i$ for $i\in[M]$. For a fixed direction $v\in\mathbb{R}^d$, let $x=\lambda v$. Assume that $v^\top\hat{\beta}_i \neq v^\top\hat{\beta}_j$ for all $i\neq j$. Then, we want to know what the predicted rank and confidence set is for each item as $\lambda \rightarrow \infty$. \\

\begin{lemma}
    Suppose $x =\lambda v, \lambda > 0$ and $\hat{\theta}_i(x)\neq\hat{\theta}_j(x)$, then
    \begin{equation*}
    \lim_{\lambda\rightarrow\infty} \hat{r}_j(x) = 1 + \sum_{i\in [M]} \mathds{1}\{v^\top\hat{\beta}_i > v^\top\hat{\beta}_j\}.
    \label{lem:extrapolation_rank}
\end{equation*}
\end{lemma}

\begin{proof}[Proof of Lemma \ref{lem:extrapolation_rank}]
    \begin{align*}
    \hat{r}_j(x) &= 1 + \sum_{i\in [M]} \mathds{1}\{\hat{\theta}_i(x) > \hat{\theta}_j(x)\}\\
    &= 1 + \sum_{i\in [M]} \mathds{1}\{\frac{\hat{\theta}_i(x)}{\lambda} > \frac{\hat{\theta}_j(x)}{\lambda}\}\\
    &= 1 + \sum_{i\in [M]} \mathds{1}\{\frac{\hat{\beta}_{0i}(x)}{\lambda} + v^\top\hat{\beta}_i > \frac{\hat{\beta}_{0j}(x)}{\lambda} + v^\top\hat{\beta}_j\}\\
    &\rightarrow 1 + \sum_{i\in [M]} \mathds{1}\{v^\top\hat{\beta}_i > v^\top\hat{\beta}_j\}.
\end{align*}
\end{proof}

\begin{lemma}
    Suppose $x = \lambda v, \lambda > 0$. Define the normalized confidence set of utility differences $C_\lambda(1-\alpha, \lambda v, (i,j)) = \frac{1}{\lambda}C(1-\alpha, \lambda v, (i,j))$. Then
        \begin{equation*}
            \lim_{\lambda\rightarrow\infty} C_\lambda(1-\alpha, \lambda v, (i,j)) = \prod_{(i,j)\in S}\left[v^\top(\hat{\beta}_i - \hat{\beta}_j) \pm t^{\infty}_{\alpha/2}\sqrt{v^\top\widehat{Var}(\hat{\beta}_i-\hat{\beta}_j)v}\right],
        \end{equation*}
        where 
        \begin{equation*}
            t_{\alpha/2}^\infty(x) = \inf\left\{z: \mathbb{P}\left(\max_{(j, k) \in S} \frac{|v^\top(\hat{\beta}_j - \hat{\beta}_k - (\beta_j - \beta_k))|}{\sqrt{v^\top\widehat{Var}(\hat{\beta}_j - \hat{\beta}_k)v}}  \geq z\right) \geq 1-\frac{\alpha}{2}\right\}.
        \end{equation*}
    \label{lem:extrapolation_score_diff_ci}
\end{lemma}

\begin{proof}[Proof of Lemma \ref{lem:extrapolation_score_diff_ci}]
    
To find the limit of $C_\lambda$ as $\lambda \rightarrow \infty$, we find the limit of each term in the confidence interval. Since all the limits inside $C_\lambda$ exist, as shown in Lemmas~\ref{lem:asymptotic_score_difference}, \ref{lem:asymptotic_std}, \ref{lem:asymptotic_max_statistic}, \ref{lem:asymptotic_quantile}, we have 
\begin{align*}
    \lim_{\lambda\rightarrow\infty} C_\lambda(1-\alpha, x, (j,k)) &= \prod_{(i,j)\in S}\lim_{\lambda\rightarrow\infty}\left[\frac{\hat{\theta}_j(x)-\hat{\theta}_k(x)}{\lambda} \pm t_{\alpha/2} (x)\frac{\hat{\sigma}_{jk}(x)}{\lambda}\right]\\
    &= \prod_{(i,j)\in S}\left[\lim_{\lambda\rightarrow\infty}\frac{\hat{\theta}_j(x)-\hat{\theta}_k(x)}{\lambda} \pm \left(\lim_{\lambda\rightarrow\infty}t_{\alpha/2}(x)\right)\left(\lim_{\lambda\rightarrow\infty}\frac{\hat{\sigma}_{jk}(x)}{\lambda}\right)\right]\\
    &= \prod_{(i,j)\in S}\left[v^\top(\hat{\beta}_i - \hat{\beta}_j) \pm t^{\infty}_{\alpha/2}\sqrt{v^\top\widehat{Var}(\hat{\beta}_i-\hat{\beta}_j)v}\right].
\end{align*}
\end{proof}

\subsubsection{Auxiliary Lemmas for the Proof of Proposition~\ref{prop:asymptotic_extrapolation}}
\begin{lemma}
    Suppose $x = \lambda v, \lambda > 0$. Then
    \begin{equation*}
        \lim_{\lambda\rightarrow\infty} \frac{\hat{\theta}_j(x)-\hat{\theta}_k(x)}{\lambda}  = v^\top(\hat{\beta}_j - \hat{\beta}_k)
    \end{equation*}
    \label{lem:asymptotic_score_difference}
\end{lemma}

\begin{proof}
    For the estimated utility difference,
\begin{align*}
    \frac{\hat{\theta}_j(x)-\hat{\theta}_k(x)}{\lambda} &= \frac{\hat{\beta}_{0j} - \hat{\beta}_{0k} + \lambda v^\top(\hat{\beta}_j - \hat{\beta}_k)}{\lambda}\\
    &=\frac{\hat{\beta}_{0j} - \hat{\beta}_{0k}}{\lambda} + v^\top(\hat{\beta}_j - \hat{\beta}_k)\\
    &\rightarrow v^\top(\hat{\beta}_j - \hat{\beta}_k).
\end{align*}
\end{proof}

\begin{lemma}
Suppose $x = \lambda v, \lambda > 0$. Then
    \begin{equation*}
        \lim_{\lambda\rightarrow\infty} \frac{\hat{\sigma}_{ij}(x)}{\lambda}  = \sqrt{v^\top\widehat{Var}(\hat{\beta}_j - \hat{\beta}_k)v},
    \end{equation*}
    where $\widehat{Var}(\hat{\beta}_j - \hat{\beta}_k)$ is an estimate of the variance of the difference of covariate parameters for items $j$ and $k$.
    \label{lem:asymptotic_std}
\end{lemma}

\begin{proof}
    For the estimated standard deviation of the difference of utilities
\begin{align*}
    \frac{\hat{\sigma}_{jk}(x)}{\lambda} &= \frac{1}{\lambda}\sqrt{\widehat{Var}(\hat{\beta}_{0j} - \hat{\beta}_{0k} + \lambda v^\top(\hat{\beta}_j - \hat{\beta}_k))} \\
    &=\frac{1}{\lambda}\sqrt{\widehat{Var}(\hat{\beta}_{0j} - \hat{\beta}_{0k}) + \lambda^2 v^\top\widehat{Var}(\hat{\beta}_j - \hat{\beta}_k)v - 2\lambda v^\top \widehat{Cov}(\hat{\beta}_{0j} - \hat{\beta}_{0k},\hat{\beta}_j - \hat{\beta}_k )}\\
    &=\sqrt{\frac{\widehat{Var}(\hat{\beta}_{0j} - \hat{\beta}_{0k})}{\lambda^2} + v^\top\widehat{Var}(\hat{\beta}_j - \hat{\beta}_k)v - \frac{2 v^\top \widehat{Cov}(\hat{\beta}_{0j} - \hat{\beta}_{0k},\hat{\beta}_j - \hat{\beta}_k )}{\lambda}}\\
    &\rightarrow \sqrt{v^\top\widehat{Var}(\hat{\beta}_j - \hat{\beta}_k)v}.
\end{align*}
\end{proof}

\begin{lemma}
Suppose $x = \lambda v, \lambda > 0$. Then
    \begin{equation*}
        \lim_{\lambda\rightarrow\infty} \mathcal{T}(x)  = \mathcal{T}_{\infty}(x) = \max_{(j, k) \in S} \frac{|v^\top(\hat{\beta}_j - \hat{\beta}_k - (\beta_j - \beta_k))|}{\sqrt{v^\top\widehat{Var}(\hat{\beta}_j - \hat{\beta}_k)v}}.
    \end{equation*}
    \label{lem:asymptotic_max_statistic}
\end{lemma}

\begin{proof}
For the max statistic
    \begin{align*}
    \mathcal{T}(x) &= \max_{(j, k) \in S} \frac{|\hat{\beta}_{0j} - \hat{\beta}_{0k} - ({\beta}_{0j} - {\beta}_{0k}) + \lambda v^\top(\hat{\beta}_j - \hat{\beta}_k - (\beta_j - \beta_k))|}{\hat{\sigma}_{jk}(x)} \\
    &= \max_{(j, k) \in S}\frac{|\frac{\hat{\beta}_{0j} - \hat{\beta}_{0k} - ({\beta}_{0j} - {\beta}_{0k})}{\lambda} + v^\top(\hat{\beta}_j - \hat{\beta}_k - (\beta_j - \beta_k))|}{|\sqrt{\frac{\widehat{Var}(\hat{\beta}_{0j} - \hat{\beta}_{0k})}{\lambda^2} + v^\top\widehat{Var}(\hat{\beta}_j - \hat{\beta}_k)v - \frac{2 v^\top \widehat{Cov}(\hat{\beta}_{0j} - \hat{\beta}_{0k},\hat{\beta}_j - \hat{\beta}_k)}{\lambda}}} \\
    &\rightarrow \max_{(j, k) \in S} \frac{|v^\top(\hat{\beta}_j - \hat{\beta}_k - (\beta_j - \beta_k))|}{\sqrt{v^\top\widehat{Var}(\hat{\beta}_j - \hat{\beta}_k)v}}\\
    &= \mathcal{T}_{\infty}(x) 
\end{align*}
\end{proof}

\begin{lemma}
Suppose $x = \lambda v, \lambda > 0$. Then
\begin{equation*}
    \lim_{\lambda\rightarrow\infty} t_{\alpha/2}(x) = t_{\alpha/2}^\infty(x),
\end{equation*}
where $t_{\alpha/2}^\infty(x)$ denotes the $(1-\alpha)$ quantile of the limiting max statistic $\mathcal{T}_\infty(x)$.
    \label{lem:asymptotic_quantile}
\end{lemma}

\begin{proof}
The statistic $\mathcal{T}_\infty(x)$ is a continuous function of $\hat{\beta}$. The distribution of the maximum has a continuous Cumulative Distribution Function. By Lemma~\ref{lem:asymptotic_max_statistic} and the Continuous Mapping Theorem
\begin{equation*}
    t_{\alpha/2}(x) = \inf\{z: \mathbb{P}(\mathcal{T}(x) \geq z) \geq 1-\frac{\alpha}{2}\}\rightarrow \inf\{z: \mathbb{P}(\mathcal{T}_\infty(x)  \geq z) \geq 1-\frac{\alpha}{2}\} = t_{\alpha/2}^\infty(x)
\end{equation*}
as $\lambda\rightarrow\infty$.
\end{proof}

\end{document}